\documentclass[11pt]{article}

\usepackage[final]{acl}

\usepackage{times}
\usepackage{latexsym}
\usepackage[T1]{fontenc}
\usepackage[utf8]{inputenc}
\usepackage{microtype}
\usepackage{inconsolata}
\usepackage{booktabs}
\usepackage{amsfonts}
\usepackage{nicefrac}
\usepackage{xcolor}
\usepackage{amsmath}
\usepackage{amssymb}
\usepackage{amsthm}
\usepackage{mathtools}
\usepackage{graphicx}
\graphicspath{{./}}
\usepackage{url}
\usepackage{subcaption}
\usepackage{wrapfig}
\usepackage{float}
\usepackage{placeins}
\floatstyle{ruled}
\newfloat{algorithm}{tbp}{loa}
\floatname{algorithm}{Algorithm}
\usepackage{multirow}
\usepackage{makecell}
\usepackage{array}
\usepackage{tabularx}
\usepackage{colortbl}
\usepackage{tikz}
\usepackage{pgfplots}
\pgfplotsset{compat=1.17}
\usetikzlibrary{shapes,arrows,positioning,fit,backgrounds,calc}
\usepgfplotslibrary{fillbetween}
\usepackage{listings}
\usepackage[skins,listings]{tcolorbox}
\usepackage{pifont}
\newcommand{\cmark}{\ding{51}}
\newcommand{\xmark}{\ding{55}}
\newtheorem{assumption}{Assumption}[section]
\newtheorem{theorem}{Theorem}[section]
\newtheorem{proposition}{Proposition}[section]
\newtheorem{corollary}{Corollary}[section]

\definecolor{langblue}{RGB}{31,119,180}
\definecolor{langorange}{RGB}{255,127,14}
\definecolor{langgreen}{RGB}{44,160,44}
\definecolor{langgray}{RGB}{127,127,127}
\definecolor{highlightrow}{RGB}{235,245,255}
\definecolor{warnrow}{RGB}{255,245,220}
\definecolor{codekeyword}{RGB}{29,78,216}
\definecolor{codestring}{RGB}{4,120,87}
\definecolor{codecomment}{RGB}{100,116,139}
\definecolor{codeidentifier}{RGB}{190,24,93}
\definecolor{codefunction}{RGB}{147,51,234}
\definecolor{codenumber}{RGB}{234,88,12}
\definecolor{codeoperator}{RGB}{220,38,38}
\definecolor{codepunct}{RGB}{37,99,235}
\lstdefinelanguage{Clojure}{
  morekeywords={defn,cond,group-by,first},
  alsoletter={:,-,?},
  morekeywords=[2]{:else},
  sensitive=true,
  morecomment=[l]{;},
  morestring=[b]"
}
\lstdefinelanguage{Rust}{
  morekeywords={fn,let,mut,for,in,return,Vec,HashMap},
  morekeywords=[2]{entry,chars,next,unwrap,or_insert,push,new,vec},
  sensitive=true,
  morecomment=[l]{//},
  morecomment=[s]{/*}{*/},
  morestring=[b]"
}
\lstdefinestyle{langselectlisting}{
  basicstyle=\ttfamily\normalsize,
  keywordstyle=\bfseries\color{codekeyword},
  keywordstyle=[2]\bfseries\color{codefunction},
  stringstyle=\color{codestring},
  commentstyle=\itshape\color{codecomment},
  identifierstyle=\color{black!88},
  numberstyle=\normalsize\color{codenumber},
  showstringspaces=false,
  columns=fullflexible,
  keepspaces=true,
  breaklines=true,
  literate=
    {*}{{{\color{codeoperator}*}}}1
    {+}{{{\color{codeoperator}+}}}1
    {-}{{{\color{codeoperator}-}}}1
    {=}{{{\color{codeoperator}=}}}1
    {<}{{{\color{codeoperator}<}}}1
    {>}{{{\color{codeoperator}>}}}1
    {:}{{{\color{codepunct}:}}}1
    {;}{{{\color{codepunct};}}}1
    {(}{{{\color{codepunct}(}}}1
    {)}{{{\color{codepunct})}}}1
    {[}{{{\color{codepunct}[}}}1
    {]}{{{\color{codepunct}]}}}1
    {\{}{{{\color{codepunct}\{}}}1
    {\}}{{{\color{codepunct}\}}}}1
    {0}{{{\color{codenumber}0}}}1
    {1}{{{\color{codenumber}1}}}1
    {2}{{{\color{codenumber}2}}}1,
  aboveskip=0pt,
  belowskip=0pt
}
\newtcblisting{langselectcodebox}[3][]{%
  listing only,
  listing engine=listings,
  listing options={style=langselectlisting,language=#2},
  enhanced,
  width=\linewidth,
  colback=#3!16!white,
  colframe=#3,
  boxrule=1.15pt,
  arc=3.0pt,
  left=4pt,
  right=4pt,
  top=3pt,
  bottom=3pt,
  boxsep=1.6pt,
  before skip=0pt,
  after skip=4pt,
  fonttitle=\bfseries\normalsize,
  coltitle=white,
  title filled,
  colbacktitle=#3!92!black,
  coltitle=white,
  drop shadow={black!18!white},
  title={#1}
}
\lstdefinestyle{langselectlistingfigone}{
  style=langselectlisting,
  basicstyle=\ttfamily\small,
  numberstyle=\small\color{codenumber}
}
\newtcblisting{langselectcodeboxfigone}[3][]{%
  listing only,
  listing engine=listings,
  listing options={style=langselectlistingfigone,language=#2},
  enhanced,
  width=\linewidth,
  colback=#3!16!white,
  colframe=#3,
  boxrule=1.15pt,
  arc=3.0pt,
  left=4pt,
  right=4pt,
  top=3pt,
  bottom=3pt,
  boxsep=1.6pt,
  before skip=0pt,
  after skip=4pt,
  fonttitle=\bfseries\small,
  coltitle=white,
  title filled,
  colbacktitle=#3!92!black,
  coltitle=white,
  drop shadow={black!18!white},
  title={#1}
}

\title{LangSelect: Cost-Aware Target-Language Routing for LLM Code Generation\thanks{Anonymous artifacts: \url{https://anonymous.4open.science/r/LangSelect-4BEB/README.md}.}}

\author{
  Son Ha Xuan\textsuperscript{1,$*$} \quad
  Phat T. Tran-Truong\textsuperscript{2,$*$} \quad
  Xuan-Bach Le\textsuperscript{2,$\dagger$} \quad
  Nghia Duong-Trung\textsuperscript{3} \\[3pt]
  \textsuperscript{1}RMIT University, Ho Chi Minh City, Vietnam \\
  \textsuperscript{2}Faculty of Computer Science and Engineering, Ho Chi Minh City University of Technology (HCMUT), \\
  VNU-HCM, Ho Chi Minh City, Vietnam \\
  \textsuperscript{3}German Research Center for Artificial Intelligence (DFKI), Berlin, Germany \\[3pt]
  \texttt{ha.son@rmit.edu.vn} \quad \texttt{\{phatttt, lexuanbach\}@hcmut.edu.vn} \quad \texttt{nghia\_trung.duong@dfki.de} \\[3pt]
  \textsuperscript{$*$}Equal contribution. \quad \textsuperscript{$\dagger$}Corresponding author.
}

\begin{document}

\maketitle

\begin{abstract}
LLM code-generation systems usually choose a target programming language before decoding and treat that choice as fixed. We show that, for language-flexible programming tasks---tasks where several target languages are acceptable and checkable by the same tests---this choice is a measurable cost lever: verified implementations of the same task can differ substantially in generated-token length. We introduce LangSelect, a verification-aware router that selects the target language before generation and falls back when the first attempt fails. To separate offline routing opportunity from end-to-end behavior, we evaluate verified-solution replay, which chooses among already accepted corpus solutions, and live GPT-5 generation, which charges every generation attempt, including failures and fallbacks. On MultiLang-Bench, a 3,000-task, 8-language verified corpus, replay shows substantial language-routing headroom. In live evaluation on 450 held-out tasks, a train-split Domain heuristic baseline reduces harness-proxy tokens, which include wrapper and entrypoint overhead, by 50.3\% at 92.9\% pass after fallback, while a learned CodeBERT+metadata selector reaches the highest pass after fallback, 93.8\%, with a 3.7\% token increase. These results show that output-language routing can define a practical cost--correctness frontier for unit-test-verifiable code generation.
\end{abstract}

\section{Introduction}
\label{sec:intro}

Modern code-generation systems increasingly operate as iterative agents: they inspect context, synthesize code, run tests, diagnose failures, and retry. SWE-bench~\cite{jimenez2023swebench} and SWE-agent~\cite{yang2024sweagent} make this loop explicit. Because each failed attempt adds generated-code tokens, cost depends on the retry trace, not only the final program. Existing efficiency work changes prompts~\cite{jiang2023llmlingua}, model cascades~\cite{chen2023frugalgpt}, or decoding~\cite{xu2025cod}, but usually fixes one pre-decoding choice: the requested output language.

For language-flexible tasks---where several target languages are acceptable and the result can be checked by the same unit tests---we ask which programming language the model should generate. Valid implementations differ in syntax, libraries, and concision~\cite{nanz2015comparative}, while tokenizers add uneven segmentation~\cite{sennrich2016bpe,kudo2018sentencepiece,bostrom2020bpe}. We therefore treat target language as inference-time routing. Task structure informs the first language choice, while verification exposes pass rate, fallback behavior, and programming-language confusion, where outputs drift toward an off-target language, often Python~\cite{moumoula2025plc}.

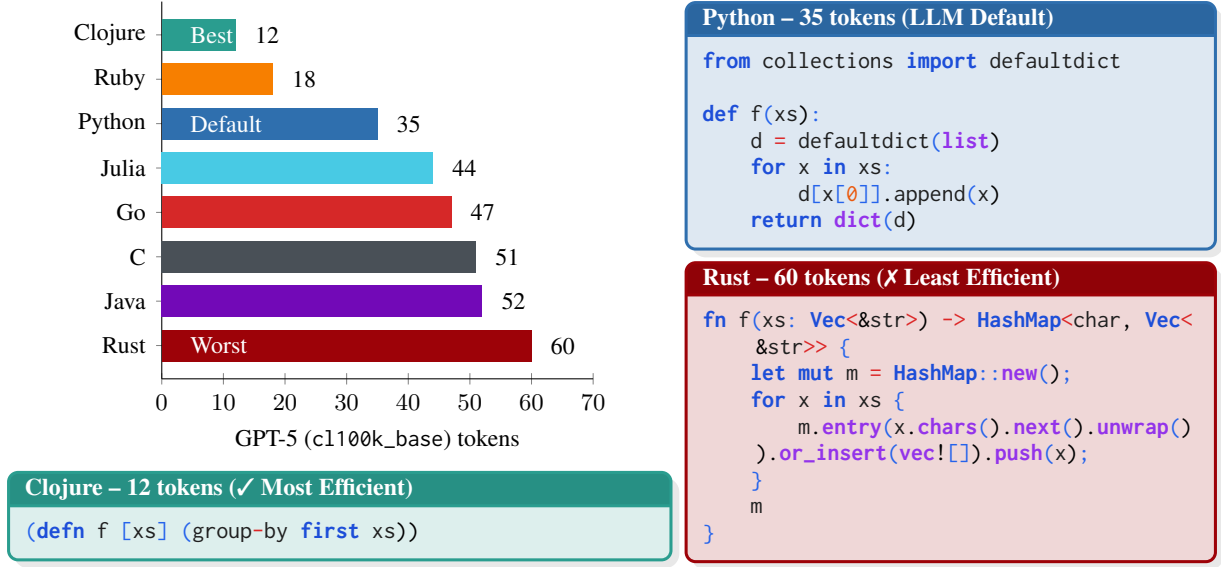
\begin{figure*}[t!]
  \centering
  \definecolor{fibclojure}{HTML}{2A9D8F}
  \definecolor{fibjulia}{HTML}{48CAE4}
  \definecolor{fibruby}{HTML}{F77F00}
  \definecolor{fibpython}{HTML}{2F6FAE}
  \definecolor{fibjava}{HTML}{7209B7}
  \definecolor{fibgo}{HTML}{D62828}
  \definecolor{fibrust}{HTML}{9D0208}
  \definecolor{fibc}{HTML}{495057}
  \begin{minipage}[t]{0.55\textwidth}
      \vspace{0pt}
      \centering
      \begin{tikzpicture}
        \begin{axis}[
          width=0.828\linewidth,
          height=6.57cm,
          xmin=0,
          xmax=70,
          xlabel={GPT-5 (\texttt{cl100k\_base}) tokens},
          ymin=-0.75,
          ymax=7.75,
          y dir=reverse,
          ytick={0,1,2,3,4,5,6,7},
          yticklabels={Clojure, Ruby, Python, Julia, Go, C, Java, Rust},
          yticklabel style={font=\small},
          xticklabel style={font=\small},
          xlabel style={font=\small},
          xtick={0,10,20,30,40,50,60,70},
          axis lines=left,
          axis line style={black, -},
          clip=false,
          xmajorgrids=false,
          ymajorgrids=false,
        ]
          \draw[fill=fibclojure, draw=none] (axis cs:0,-0.36) rectangle (axis cs:12,0.36);
          \draw[fill=fibruby, draw=none]    (axis cs:0,0.64) rectangle (axis cs:18,1.36);
          \draw[fill=fibpython, draw=none]  (axis cs:0,1.64) rectangle (axis cs:35,2.36);
          \draw[fill=fibjulia, draw=none]   (axis cs:0,2.64) rectangle (axis cs:44,3.36);
          \draw[fill=fibgo, draw=none]      (axis cs:0,3.64) rectangle (axis cs:47,4.36);
          \draw[fill=fibc, draw=none]       (axis cs:0,4.64) rectangle (axis cs:51,5.36);
          \draw[fill=fibjava, draw=none]    (axis cs:0,5.64) rectangle (axis cs:52,6.36);
          \draw[fill=fibrust, draw=none]    (axis cs:0,6.64) rectangle (axis cs:60,7.36);

          \node[font=\small, text=white, anchor=west, align=left] at (axis cs:3,0) {Best};
          \node[font=\small, text=white, anchor=west, align=left] at (axis cs:3,2) {Default};
          \node[font=\small, text=white, anchor=west, align=left] at (axis cs:3,7) {Worst};
          \node[font=\small, anchor=west] at (axis cs:13.5,0) {12};
          \node[font=\small, anchor=west] at (axis cs:19.5,1) {18};
          \node[font=\small, anchor=west] at (axis cs:36.5,2) {35};
          \node[font=\small, anchor=west] at (axis cs:45.5,3) {44};
          \node[font=\small, anchor=west] at (axis cs:48.5,4) {47};
          \node[font=\small, anchor=west] at (axis cs:52.5,5) {51};
          \node[font=\small, anchor=west] at (axis cs:53.5,6) {52};
          \node[font=\small, anchor=west] at (axis cs:61.5,7) {60};
        \end{axis}
      \end{tikzpicture}
      \par\vspace{4pt}
      \begin{langselectcodeboxfigone}[Clojure -- 12 tokens (\cmark\ Most Efficient)]{Clojure}{fibclojure}
(defn f [xs] (group-by first xs))
      \end{langselectcodeboxfigone}
  \end{minipage}\hfill%
  \begin{minipage}[t]{0.44\textwidth}
      \vspace{0pt}
      \begin{langselectcodeboxfigone}[Python -- 35 tokens (LLM Default)]{Python}{fibpython}
from collections import defaultdict

def f(xs):
    d = defaultdict(list)
    for x in xs:
        d[x[0]].append(x)
    return dict(d)
      \end{langselectcodeboxfigone}
      \begin{langselectcodeboxfigone}[Rust -- 60 tokens (\xmark\ Least Efficient)]{Rust}{fibrust}
fn f(xs: Vec<&str>) -> HashMap<char, Vec<&str>> {
    let mut m = HashMap::new();
    for x in xs {
        m.entry(x.chars().next().unwrap()).or_insert(vec![]).push(x);
    }
    m
}
      \end{langselectcodeboxfigone}
  \end{minipage}
  \caption{\small A high-spread task in the 8-language portfolio. For this task, output-language choice creates a $5\times$ generated-code token spread, compared with a dataset-wide mean of $2.19\times$ (95\% CI $[2.17,2.20]$). Tokens are counted with GPT-5's \texttt{tiktoken cl100k\_base} encoder; default denotes Python-only, and efficient denotes the shortest tokenized accepted function body.}
  \label{fig:token_variation}
\end{figure*}

Figure~\ref{fig:token_variation} illustrates the opportunity: accepted solutions to one group-by task range from $12$ tokens in Clojure to $60$ in Rust. In the $3{,}000$-task, $8$-language MultiLang-Bench corpus, cheapest and most expensive accepted implementations differ by $2.19\times$ on average (95\% CI $[2.17,2.20]$). The example exposes the lever; the evaluation asks whether a selector can recover that token-spread opportunity at scale.

We instantiate this idea in \textbf{LangSelect}. At a high level, LangSelect ranks target languages from pre-generation task features, prompts generation in the top language, verifies the result, and retries on failure. The selector combines task-family priors, CodeBERT embeddings and metadata, uncertainty-aware language heads with LinUCB-style contextual bandits~\cite{li2010linucb,zhou2020neural}, and verification-gated fallback. We separate \emph{verified-solution replay headroom}, an audit that chooses among already accepted corpus solutions, from \emph{live targeted generation}, which samples fresh GPT-5 completions in the selected language and charges every generation attempt, including failures and fallbacks.

The replay and live numbers answer different questions. Under function-body accounting, replay reduces cost by $28.3\%$--$34.2\%$ on HumanEval-X, MultiPL-E, and McEval, and by $42.1\%$ on MultiLang-Bench; under the more conservative harness-proxy boundary, which includes wrapper and entrypoint overhead, the same MultiLang-Bench oracle gives $19.09\%$. In live evaluation on $450$ held-out GPT-5 tasks, the Domain heuristic baseline, a static training-split domain rule, achieves $+50.3\%$ harness-proxy savings at $92.9\%$ pass after fallback, close to the non-deployable Hindsight oracle row ($+49.6\%$ / $92.7\%$). CodeBERT+metadata, a learned selector using CodeBERT task embeddings plus metadata, reaches the best pass after fallback ($93.8\%$) while incurring a $3.7\%$ harness-proxy token increase. Both policies add under $1\%$ latency; a mainstream-only portfolio (Python, Java, Go, Ruby) still retains $15.15\%$ replay headroom.

We make four contributions:
\begin{enumerate}
  \item We formulate target-language routing under a verifier with an explicit token-cost and pass-rate boundary.
  \item We introduce MultiLang-Bench, a verified $3{,}000$-task, $8$-language corpus with deterministic splits and token, character, and byte accounting checks.
  \item We present LangSelect as a selector framework compatible with heuristic, supervised, and bandit policies, with priors, contextual correction, uncertainty-aware heads, verification-gated fallback, and deployment constraints.
  \item We evaluate replay and $450$ live GPT-5 tasks with baselines, fallback behavior, latency, accounting sensitivity, and portfolio constraints.
\end{enumerate}

LangSelect complements model-cascade routing~\cite{dohan2022cascades,ong2024routellm}: cascades vary the model, while LangSelect varies the target language. It also differs from multilingual code benchmarks~\cite{chen2021codex,austin2021mbpp,cassano2023multipl}, which ask whether models solve tasks across languages; we ask which acceptable language to attempt first.

\section{Motivating Evidence}
\label{sec:variation}

Our MultiLang-Bench corpus shows that this routing opportunity extends well beyond a single illustrative example. Across $3{,}000$ tasks and $8$ languages, the cheapest and most expensive unit-test-verified function bodies differ by $2.19\times$ on average under the \texttt{cl100k\_base} encoder (95\% CI $[2.17,2.20]$). Character- and byte-level audits show similar variation, indicating that the effect is not merely an artifact of tokenization.

The spread is workload-dependent. Appendix Figure~\ref{fig:token_heatmap} breaks the corpus into six $300$-task families: row-wise median spreads range from $1.33\times$ to $1.50\times$, and the cheapest language varies by family, often between Go and Python. Measured cheapest-language headroom is largest for math, string, I/O-converted tasks (input/output benchmark tasks converted into pure functions), and HumanEval/MBPP/APPS-style tasks, and smallest for numeric tasks (Section~\ref{sec:results}). This pattern has two implications. First, language choice is task-dependent rather than globally fixed, which motivates contextual routing. Second, the attainable savings depend on workload mix: numeric-heavy deployments should expect smaller gains than workloads dominated by string processing or I/O-style tasks. We revisit this stratification in the live results.

\section{Problem Formulation}
\label{sec:problem}

We formalize target-language routing to make the comparison boundary explicit: the fixed generator, allowed languages, verifier, and token-count rule used for a comparison. Given a fixed language portfolio $\mathcal{L}=\{\ell_1,\ldots,\ell_K\}$, each task $t$ consists of a specification $s_t$ and a test suite $\mathcal{U}_t$. For a generator $\mathcal{G}$, let $c(s_t,\ell)$ denote the expected number of generated tokens when $\mathcal{G}$ is prompted to produce language $\ell$ under the stated accounting rule, and let $q(s_t,\ell)$ denote strict pass@1: the first generated attempt passes all tests and language-fidelity checks under the verifier. A routing policy $\pi:\mathcal{S}\to\mathcal{L}$ then selects the first language to attempt, with the goal of reducing generated-code cost while maintaining the empirical pass boundary used throughout the paper. We use Python-only as the pass-rate reference because Python is what the model defaults to when no language is requested. Let $q_{\mathrm{py}}$ denote this Python-only pass rate and $\epsilon$ the allowed shortfall; we ask the policy to minimize expected cost subject to a Python-relative pass constraint:
\begin{equation}
  \begin{aligned}
  \min_{\pi}\quad & \mathbb{E}_{t}\bigl[c(s_t, \pi(s_t))\bigr] \\
  \text{s.t.}\quad &
  \mathbb{E}_{t}\bigl[q(s_t, \pi(s_t))\bigr]\geq q_{\mathrm{py}}-\epsilon .
  \end{aligned}
  \label{eq:objective}
\end{equation}
The two evaluation protocols differ in how they treat $\epsilon$. In verified-solution replay, $\epsilon=0$ by construction because the policy chooses among already-passing corpus solutions and cannot lose pass rate. In live targeted generation, $\epsilon$ is observed rather than assumed: it is the pass-rate shortfall the policy incurs, and we report it alongside cost. The attainable cost/pass pairs over a fixed generator, portfolio, verifier, and token-count accounting boundary form the \emph{empirical verified language-routing frontier}. The hindsight oracle picks the cheapest passing language per task within the verified corpus, so we treat it as a within-corpus reference rather than a universal upper bound.

\paragraph{Fallback reward.}
The objective is defined at the task level, but the selector receives feedback at the attempt level because a failed first choice may still be recovered through fallback. For task $t$ and attempt $a$, the selector encodes the specification as $x_t=\phi(s_t)$ and selects a language $\ell_{t,a}$. The generator then returns a program $y_{t,a}$ with token count $c_{t,a}$, and the selector receives the reward
\begin{equation}
  r_{t,a} = \bigl(1-\tau(c_{t,a})\bigr) \cdot
    \mathbf{1}\!\bigl[\mathrm{pass}(y_{t,a},\,\mathcal{U}_{t})\bigr],
  \label{eq:reward}
\end{equation}
where $\tau(\cdot)$ maps token counts to $[0,1]$ using percentiles computed from the warm-start corpus. This reward gives higher credit to passing attempts that use fewer tokens, while assigning zero credit to failures. The selector is therefore encouraged to choose languages that are both inexpensive and likely to pass. When fallback succeeds, credit is assigned only to the language that passes; earlier failed attempts still contribute to the task's token cost. Appendix~\ref{app:method_details} gives the full equations, warm-start procedure, and regret discussion for the idealized LinUCB-style update model.

\section{LangSelect Framework}
\label{sec:framework}

LangSelect (Figure~\ref{fig:framework_overview}) implements Eq.~\eqref{eq:objective} as a verification-guided routing loop. For each task, it encodes the specification, scores the available languages, generates code with a language-conditioned prompt, verifies the result with unit tests and language-fidelity checks, and records the observed outcome for policies that update after verification. Verification remains inside the loop because prompt-level features alone cannot reliably anticipate programming-language confusion or language-specific failures. Executing the generated program therefore provides the most direct empirical signal of whether a routing choice was actually successful.

\begin{figure*}[!t]
  \centering
  \includegraphics[width=0.98\textwidth]{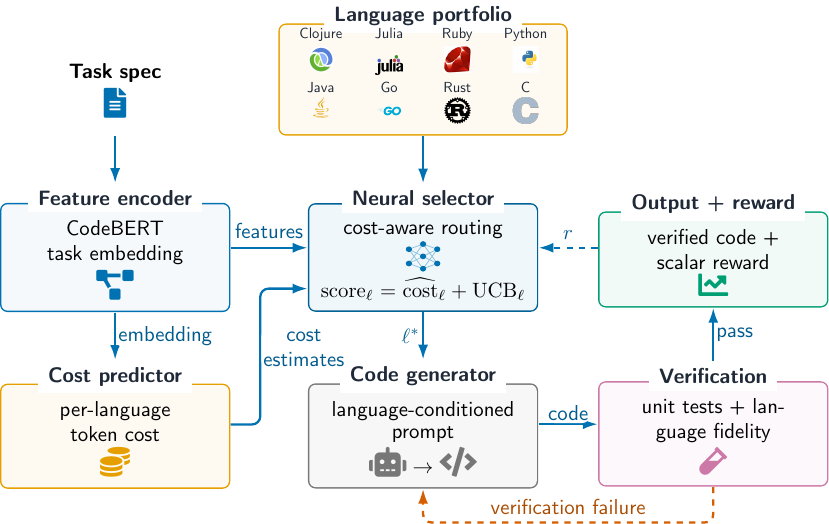}
  \caption{LangSelect as a verified output-language routing loop. The selector ranks target languages from task features, generation is checked by unit tests and language-fidelity checks, and verification failures trigger fallback attempts and reward updates.}
  \label{fig:framework_overview}
\end{figure*}

The selector moves from coarse signals to task-specific evidence. It first uses deployable metadata available before generation, such as domain, source group (for example, HumanEval/MBPP/APPS-style versus Codeforces/LeetCode-style), argument/return type family, and difficulty, to set a coarse prior; the Domain heuristic baseline is the simplest version of this layer and already provides a strong baseline. It then combines CodeBERT embeddings with structured metadata to refine the prior when the task representation points to a different language. Per-language Neural-LinUCB heads add uncertainty-aware selection, allowing the system to explore less-used languages while exploiting reliable ones. Finally, verification-gated fallback retries languages in descending uncertainty-aware score order and charges every failed attempt, making routing errors recoverable rather than terminal.

We use token-spread features, computed from already generated cross-language token counts, only for auditing and not for deployment; including them in the policy would conflate prediction with information observed after generation. If the first attempt fails, LangSelect tries the next-ranked language, for up to three total attempts, with Python reserved as the final fallback. Failed languages receive zero reward, and successful credit is assigned only to the language that passes verification. Appendix~\ref{app:method_details} gives the full equations and warm-start details.

\paragraph{Deployment-constraint scoring.}
Real deployments rarely allow an unconstrained language portfolio. Maintainability requirements, integration costs, and organizational policies often restrict which languages can be used. We therefore extend the selector score to support both hard portfolio masks and soft preference penalties:
\begin{equation}
\label{eq:deploy_score}
\begin{aligned}
\operatorname{score}(\ell) =&\;
  \underbrace{\hat\theta_\ell^{\top} x_t
  + \alpha_n \sqrt{x_t^{\top} A_\ell^{-1} x_t}}_{\text{UCB}(\ell)}\\
& + \lambda_q\, \hat{q}(s_t, \ell)
  - \lambda_m\, M(\ell)
  - \lambda_i\, I(t, \ell).
\end{aligned}
\end{equation}
Here $\hat\theta_\ell$ and $A_\ell$ are the per-language LinUCB reward estimate and design matrix, $\alpha_n$ controls exploration, and $\hat{q}(s_t,\ell)$ estimates pass probability. A hard mask restricts the policy to $\pi(s_t)\in A_t\subseteq\mathcal{L}$. If the highest-scoring language is not allowed, the selector chooses the best remaining language in~$A_t$. The optional quality term $\hat{q}(s_t,\ell)$ favors languages predicted to pass, while $M(\ell)\in[0,1]$ and $I(t,\ell)\in[0,1]$ penalize organization-level maintainability cost and task-specific integration cost supplied at deployment time. The weights $\lambda_q,\lambda_m,\lambda_i\geq 0$ can be set at deployment time; with no hard mask and these weights set to zero, the score recovers the unconstrained selector. Section~\ref{sec:res_constraints} quantifies the headroom lost under common deployment restrictions.

\section{Experimental Setup}
\label{sec:setup}

\paragraph{Corpus.}
We evaluate primarily on MultiLang-Bench, a harness-verified corpus of $3{,}000$ deterministic pure-function tasks. Each task has an accepted solution in eight languages: C, Clojure, Go, Java, Julia, Python, Ruby, and Rust. The portfolio covers low-level languages (C, Rust), mainstream JVM/Go languages (Java, Go), mainstream scripting languages (Python, Ruby), and concise functional or numeric languages (Clojure, Julia). The corpus contains $990$ HumanEval/MBPP/APPS-style tasks, $1{,}230$ Codeforces/LeetCode-style tasks, and $780$ manually curated or author-reviewed tasks. We use a fixed $2{,}100/450/450$ train/validation/test split, and all $24{,}000$ accepted solutions pass their local unit-test and language-fidelity harnesses. Appendix~\ref{app:bench} gives the construction details.

We use two evaluation protocols and keep their roles distinct. \emph{Verified-solution replay} selects among implementations that have already passed the harness, so strict pass@1 is $100\%$ by construction. This protocol isolates the routing decision and cost accounting from generation noise; its headroom is the cost reduction from choosing among accepted solutions relative to Python-only under a stated accounting boundary. \emph{Live targeted generation} instead samples fresh GPT-5 completions in the selected language, runs them through the local harness, charges every generation attempt including failures and fallbacks, and reports pass after fallback. Replay estimates verified-portfolio headroom, while live evaluation measures how much a deployable selector can recover.

All reported generations use GPT-5 through the OpenAI API. We count tokens with the \texttt{tiktoken cl100k\_base} encoder and report three accounting boundaries: function body, function plus signature, and harness proxy. The harness-proxy boundary also charges the per-language wrapper and entrypoint overhead, making it the closest of the three to a deployed inference path. We therefore use it as the headline boundary, although it yields the most conservative savings.

\paragraph{Baselines.}
Every deployable baseline returns an ordered list of target languages before generation, using only training/validation data and task fields available before generation. The rows share the same prompts, GPT-5 settings, verifier, token counter, and fallback budget; only the routing policy changes. Python-only fixes the first language to Python. Static heuristics rank languages by training-split verified cost within a source group or task domain: Source heuristic uses source group, while Domain heuristic uses task domain. Learned supervised selectors train on the same train/validation split to predict the cheapest verified language from CodeBERT embeddings, metadata, or both. The Hybrid prior+LinUCB policy initializes from the domain prior and lets the validation-tuned LinUCB score override it when another language receives a higher uncertainty-aware score. The Hindsight oracle row is not deployable: it uses the test task's verified corpus costs to choose the cheapest target language and serves only as a measured reference.

For replay, we compare Python-only, Domain heuristic, CodeBERT-only, CodeBERT+metadata, Hybrid prior+LinUCB, and Hindsight oracle. Live targeted generation uses four rows: Python-only, Domain heuristic, CodeBERT+metadata, and Hindsight oracle. Section~\ref{sec:results} also reports six supervised classifiers on the same deployable features as a supervised comparison set. Fallback permits at most three attempts and charges every failed, fallback, and final successful generated token. Appendix~\ref{app:baseline_construction} gives the construction details for each comparison object.

Here, ``verified'' refers to acceptance under our unit-test and language-fidelity harness. Replay therefore isolates routing over accepted corpus solutions, rather than over all possible programs. Accepted solutions may come from different construction paths, such as translation, curation, or author review, but all are admitted only after passing the full harness and language-fidelity checks. Hyperparameters, execution controls, and the reproducibility checklist appear in Appendix~\ref{app:repro}.

\section{Results}
\label{sec:results}

We report replay headroom, supervised baselines, live targeted generation, fallback behavior, sensitivity, and portfolio constraints. Replay measures the cost opportunity under fixed correctness by choosing among verified solutions; live generation measures how much of that opportunity remains after every fresh GPT-5 generation attempt, including failures and fallbacks, is charged.

\subsection{Replay headroom}
\label{sec:res_replay}

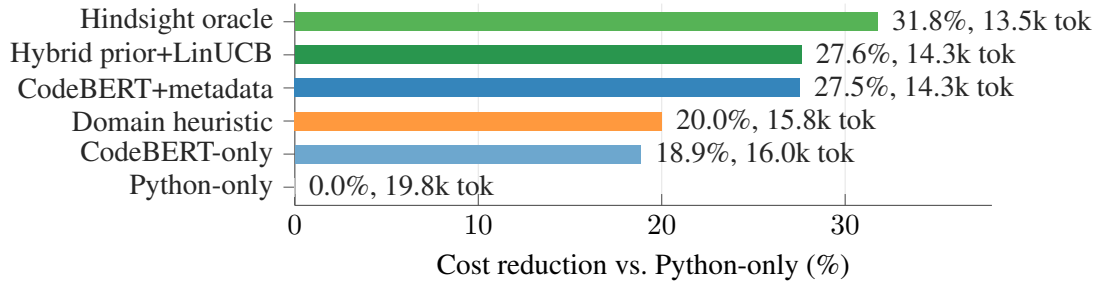
\begin{figure*}[t]
  \centering
  \begin{tikzpicture}
\begin{axis}[
  width=10.8cm,
  height=4.25cm,
  xbar,
  xmin=0,
  xmax=38,
  ymin=0.45,
  ymax=6.55,
  bar width=7pt,
  xlabel={Cost reduction vs.\ Python-only (\%)},
  ytick={1,2,3,4,5,6},
  yticklabels={Python-only,CodeBERT-only,Domain heuristic,CodeBERT+metadata,Hybrid prior+LinUCB,Hindsight oracle},
  yticklabel style={font=\normalsize, text=black!85},
  xtick={0,10,20,30},
  xticklabel style={font=\normalsize},
  xlabel style={font=\normalsize},
  xmajorgrids,
  grid style={black!10},
  axis x line*=bottom,
  axis y line*=left,
  axis line style={black!70},
  tick style={black!70},
  nodes near coords,
  nodes near coords align={horizontal},
  nodes near coords style={font=\normalsize, anchor=west, xshift=1.5pt, text=black!85},
  point meta=explicit symbolic,
  every axis plot/.append style={draw=none, bar shift=0pt},
  clip=false,
]
\addplot[fill=langgray!45] coordinates {(0.00,1) [0.0\%, 19.8k tok]};
\addplot[fill=langblue!65] coordinates {(18.86,2) [18.9\%, 16.0k tok]};
\addplot[fill=langorange!80] coordinates {(20.02,3) [20.0\%, 15.8k tok]};
\addplot[fill=langblue!90] coordinates {(27.54,4) [27.5\%, 14.3k tok]};
\addplot[fill=langgreen!75!langblue] coordinates {(27.63,5) [27.6\%, 14.3k tok]};
\addplot[fill=langgreen!80] coordinates {(31.78,6) [31.8\%, 13.5k tok]};
\end{axis}
\end{tikzpicture}

  \caption{Held-out verified-solution replay on the 450-task MultiLang-Bench test split (function-body accounting, GPT-5). Pass@1 is $100\%$ by construction; labels show cost reduction relative to Python-only.}
  \label{fig:main_results_plot}
\end{figure*}

Replay separates routing from generation noise. Each policy selects among verified solutions for the same $450$ held-out tasks, making strict pass@1 $100\%$ by construction; selector error appears only as extra cost.

\paragraph{External benchmarks.}
A similar token-cost opportunity appears beyond our corpus. On three independent multilingual benchmarks, the Hybrid prior+LinUCB policy reduces cost relative to Python-only by $28.3\%$ on HumanEval-X, $31.6\%$ on MultiPL-E, and $34.2\%$ on McEval under function-body accounting (all $p<0.001$ after Bonferroni correction; Appendix~\ref{app:supp_analysis}). These benchmarks differ in task source and language coverage, so the repeated effect suggests that generated-code cost is not only an artifact of MultiLang-Bench construction.

\paragraph{MultiLang-Bench.}
On MultiLang-Bench, the Hindsight oracle row achieves a $42.1\%$ function-body reduction (95\% CI $[39.8, 44.4]$).\footnote{Under harness-proxy accounting, which includes wrappers, signatures, and harness overhead, the same oracle gives $19.09\%$ (Table~\ref{tab:constrained_portfolio}); the live headline (\S\ref{sec:res_live}) uses this conservative boundary.} The Hybrid prior+LinUCB policy recovers most of the deployable replay headroom shown in Figure~\ref{fig:main_results_plot}. Gains are workload-sensitive: math and string tasks exceed $41\%$, I/O-converted tasks reach $39.4\%$, and numeric tasks reach $8.8\%$ (Appendix Table~\ref{tab:headroom_partitions}). Thus, replay quantifies the cost opportunity in the verified distribution, while live generation tests how much remains after every generated attempt is charged.

\subsection{Supervised baselines}
\label{sec:res_ablation}

We train six supervised baselines on CodeBERT embeddings and task metadata, reporting Wilson 95\% CIs. Random forest performs best, with $+5.83\%$ cost reduction and $61.5\%$ oracle recovery, the fraction of tasks whose first selected language matches the hindsight-optimal language (CI $[45.9, 75.1]\%$). Supervised routing is feasible on these features, but a wrong first language in replay is charged as extra cost rather than retried. LangSelect uses the same signals with uncertainty-aware exploration and verification-gated fallback, so selector mistakes can become second attempts rather than silent cost regressions; Appendix Table~\ref{tab:supervised_baselines} gives the full results.

\subsection{Live targeted generation}
\label{sec:res_live}

\begin{table*}[t]
  \centering
  \caption{Live targeted generation on the 450-task held-out split (GPT-5). Each policy generates in its chosen language and falls back on failure (up to 3 attempts; all attempts charged). Pass@1: first-language success; FB rate: fraction of tasks triggering fallback; Pass-FB: pass after fallback; SelLat: selector latency.}
  \label{tab:live_eval_main}
  \normalsize
  \setlength{\tabcolsep}{4pt}
  \renewcommand{\arraystretch}{1.12}
  \begin{tabular*}{\textwidth}{@{\extracolsep{\fill}}lrrrrr@{}}
    \toprule
    \textbf{Policy} & \textbf{Pass@1} & \textbf{FB rate} & \textbf{Pass-FB} & \textbf{Attempts} & \textbf{SelLat (ms)} \\
    \midrule
    Python-only       & 92.0\% &  0.0\% & 92.0\% & 1.00 & 0.01 \\
    Domain heuristic  & 87.3\% & 12.4\% & \textbf{92.9\%} & 1.20 & 0.01 \\
    CodeBERT+metadata & 91.8\% &  8.0\% & \textbf{93.8\%} & 1.14 & 0.94 \\
    Hindsight oracle  & 87.1\% & 12.7\% & 92.7\% & 1.20 & 0.01 \\
    \bottomrule
  \end{tabular*}
\end{table*}

\begin{table*}[t]
  \centering
  \caption{Net cost per task and reduction vs.\ Python-only under three accounting boundaries; all attempts (including failed and fallback) are charged. Harness-proxy is the headline boundary.}
  \label{tab:live_eval_cost}
  \normalsize
  \setlength{\tabcolsep}{4pt}
  \renewcommand{\arraystretch}{1.14}
  \begin{tabular*}{\textwidth}{@{\extracolsep{\fill}}lrrrrrr@{}}
    \toprule
    & \multicolumn{2}{c}{\textbf{Function body}} & \multicolumn{2}{c}{\textbf{Function + signature}} & \multicolumn{2}{c}{\textbf{Harness proxy}} \\
    \cmidrule(lr){2-3}\cmidrule(lr){4-5}\cmidrule(lr){6-7}
    \textbf{Policy} &
      \textbf{tok/task} & \textbf{Reduction} &
      \textbf{tok/task} & \textbf{Reduction} &
      \textbf{tok/task} & \textbf{Reduction} \\
    \midrule
    Python-only       & 130.7 &  +0.0\% & 139.5 &  +0.0\% & 141.4 &  +0.0\% \\
    Domain heuristic  &  47.6 & \textbf{+63.6\%} &  70.2 & \textbf{+49.7\%} &  70.3 & \textbf{+50.3\%} \\
    CodeBERT+metadata & 132.7 & $-1.5\%$ & 145.0 & $-3.9\%$ & 146.7 & $-3.7\%$ \\
    Hindsight oracle  &  47.9 & +63.4\% &  70.8 & +49.2\% &  71.2 & +49.6\% \\
    \bottomrule
  \end{tabular*}
\end{table*}

Tables~\ref{tab:live_eval_main} and~\ref{tab:live_eval_cost} report live GPT-5 evaluation on all $450$ held-out tasks. Under harness-proxy accounting, the Domain heuristic baseline achieves $+50.3\%$ cost reduction at $92.9\%$ pass after fallback, close to the non-deployable Hindsight oracle row ($+49.6\%$ / $92.7\%$). Because the Domain heuristic baseline is a static train-split rule, stable workloads can recover the low-cost operating point without online learning infrastructure.

The live $+50.3\%$ reduction exceeds the $19.09\%$ harness-proxy replay number because the denominators differ. Replay compares verified corpus solutions against the verified Python corpus solution for each task. Live evaluation instead compares all charged GPT-5 attempts against fresh GPT-5 Python generations, often with boilerplate, comments, and imports, all of which are charged. The live number is therefore the proper like-for-like comparison for the end-to-end protocol.

CodeBERT+metadata occupies the opposite operating point on the frontier. It reaches the highest pass after fallback ($93.8\%$) and incurs a $3.7\%$ harness-proxy token increase relative to Python-only for that gain. Figure~\ref{fig:cost_correctness_frontier} shows the resulting deployment choice: cost-critical settings favor the Domain heuristic baseline, while pass-rate-critical settings favor CodeBERT+metadata.

\begin{figure}[t]
  \centering
  \begin{tikzpicture}
    \begin{axis}[
      width=0.78\linewidth,
      height=5.8cm,
      xlabel={Harness-proxy tokens per task (lower is better)},
      ylabel={Pass after fallback (\%)},
      xmin=40,xmax=160,
      ymin=91.5,ymax=94.5,
      xtick={50,70,90,110,130,150},
      ytick={92,92.5,93,93.5,94},
      grid=both,
      grid style={dashed,gray!30},
      every axis label/.append style={font=\normalsize},
      tick label style={font=\normalsize},
      legend style={font=\normalsize,at={(0.98,0.02)},anchor=south east,draw=none,fill=white,fill opacity=0.85,text opacity=1},
      clip=true,
    ]
      \addplot[langorange, very thick, dashed] coordinates {(70.3,92.9) (146.7,93.8)};
      \addplot[only marks, mark=*, mark size=2.8pt, langgray]
        coordinates {(141.4,92.0)};
      \node[font=\normalsize, anchor=east, fill=white, fill opacity=0.92, text opacity=1, inner sep=1pt]
        at (axis cs:133.0,92.05) {Python-only};
      \addplot[only marks, mark=triangle*, mark size=3.2pt, langgreen]
        coordinates {(71.2,92.7)};
      \node[font=\normalsize, anchor=north west, fill=white, fill opacity=0.92, text opacity=1, inner sep=1pt]
        at (axis cs:44.0,92.60) {Hindsight oracle};
      \addplot[only marks, mark=square*, mark size=3.2pt, langblue]
        coordinates {(70.3,92.9)};
      \node[font=\normalsize\bfseries, anchor=south west, fill=white, fill opacity=0.92, text opacity=1, inner sep=1pt]
        at (axis cs:44.0,93.10) {Domain heuristic};
      \addplot[only marks, mark=diamond*, mark size=3.6pt, langblue]
        coordinates {(146.7,93.8)};
      \node[font=\normalsize\bfseries, anchor=north east, align=right, fill=white, fill opacity=0.92, text opacity=1, inner sep=1pt]
        at (axis cs:145.0,94.42) {CodeBERT+\\metadata};
    \end{axis}
  \end{tikzpicture}
\caption{Cost--correctness frontier on 450 live GPT-5 tasks. Domain heuristic is lowest-cost ($70$ tokens, $92.9\%$ pass after fallback); CodeBERT+metadata is highest-pass ($147$ tokens, $93.8\%$ pass after fallback).}
  \label{fig:cost_correctness_frontier}
\end{figure}
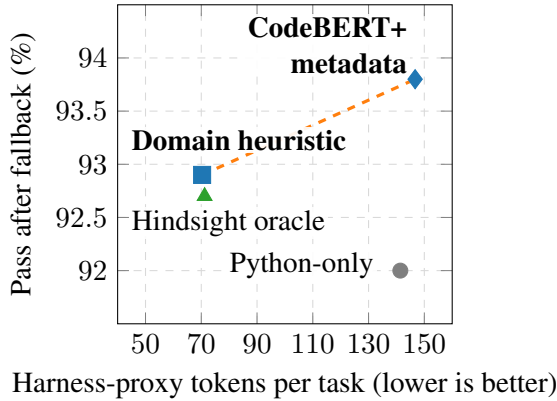\vspace{-1em}

\paragraph{Fallback behavior.}
Fallback turns selector errors into charged recovery attempts. CodeBERT+metadata gains the most from fallback ($+2.0$ pp) while invoking it least often ($8.0\%$ of tasks), suggesting both a stronger first choice and a useful second choice. The Domain heuristic baseline closely matches Hindsight oracle's fallback profile ($12.4\%$ vs. $12.7\%$ fallback rate; $1.20$ vs. $1.20$ mean attempts), which helps explain why its harness-proxy savings fall within $1$ pp of the measured reference.

\subsection{Sensitivity, latency, and cold start}
\label{sec:res_sensitivity}

\paragraph{Overhead, boundary, and cold start.}
Routing adds little overhead: $39.65$ ms mean / $50.63$ ms p95 on CPU, below $1.01\%$ of GPT-5 generation latency (Appendix Table~\ref{tab:selector_latency}); CodeBERT embedding accounts for $99.3\%$ and GPU execution cuts this by ${\approx}5\times$. The accounting boundary changes magnitude but not direction --- Domain heuristic gives $+63.6\%$ (function-body), $+49.7\%$ (function+signature), and $+50.3\%$ (harness-proxy) --- and five cost views and seven reward normalizations leave the selected language unchanged per task (Appendix Tables~\ref{tab:reward_sensitivity}--\ref{tab:reward_norm_sensitivity}, \ref{tab:warmstart_sensitivity}); we use harness-proxy as the headline boundary because it is conservative and yields a more diverse oracle. In cold start, Neural-LinUCB regret stabilizes after roughly $50$ tasks (Appendix Figure~\ref{fig:regret}), so new distributions can collect about $50$ verified observations before enabling the bandit; stable workloads use the Domain heuristic baseline directly, and fallback rate provides a practical drift signal.

\subsection{Deployment Constraints}
\label{sec:res_constraints}

\begin{table}[t]
  \centering
  \caption{Portfolio constraint sensitivity ($n{=}377$ replay, GPT-5, harness proxy). Dropping C/Rust costs $0.1$ pp; a mainstream portfolio retains $15$ pp hindsight headroom.}
  \label{tab:constrained_portfolio}
  \normalsize
  \setlength{\tabcolsep}{2pt}
  \begin{tabularx}{\linewidth}{@{}p{0.27\linewidth}Xr@{}}
    \toprule
    \textbf{Portfolio} & \textbf{Languages} & \makecell{\textbf{Hindsight}\\\textbf{headroom}} \\
    \midrule
    Full 8                  & C, Clojure, Go, Java, Julia, Python, Ruby, Rust & \textbf{19.09\%} \\
    No C/Rust & Clojure, Go, Java, Julia, Python, Ruby       & 19.00\% \\
    Mainstream              & Python, Java, Go, Ruby                          & 15.15\% \\
    \bottomrule
  \end{tabularx}
\end{table}\vspace{-0.5em}

Real deployments restrict the language portfolio for maintainability, integration, or policy reasons; Eq.~\ref{eq:deploy_score} adds hard masks (allowed set) and soft penalties (maintainability and integration costs) to the selector score. Table~\ref{tab:constrained_portfolio} shows the headroom retained: the full $8$-language portfolio gives $19.09\%$ harness-proxy headroom, removing C and Rust costs only $0.1$ pp, and a mainstream portfolio (Python, Java, Go, Ruby) still retains $15.15\%$. Soft penalties with $\lambda_m \in \{0.3, 1.0\}$ leave replay saving nearly unchanged because the oracle is mostly outside the penalized languages.

\section{Analysis}
\label{sec:analysis}

On this task distribution, the Domain heuristic baseline captures most of the live cost saving, showing that coarse domain structure already identifies many low-cost language choices; the bandit machinery mainly extends this rule for cold-start diagnostics rather than replacing the stable-domain rule in the live headline. Verification-gated fallback converts selector errors into recoverable attempts, and Eq.~\ref{eq:deploy_score} ranks languages under maintainability and integration penalties. Random forest recovers a measurable portion of the replay oracle but optimizes a single replay decision, while the deployable policy must also decide how much pass-rate risk to take on the first language and how many fallback attempts to reserve --- this is why Domain heuristic and CodeBERT+metadata occupy different useful points rather than one replacing the other. The cost-optimal language often reflects language-specific strengths (Ruby for string-heavy tasks, Clojure for math, C for systems-style code, Go or Java for structured tasks; Appendix Table~\ref{tab:heuristic_vs_langselect}). CodeBERT+metadata triggers fallback least often and receives the largest pass-rate gain from fallback, while the Domain heuristic baseline closely follows the Hindsight oracle fallback profile; routing overhead remains negligible (Appendix Table~\ref{tab:selector_latency}).

\section{Related Work}
\label{sec:related}

Cost-aware LLM inference typically changes the prompt or the model call --- prompt compression~\cite{jiang2023llmlingua}, terse reasoning~\cite{xu2025cod}, model routing~\cite{chen2023frugalgpt,ong2024routellm,dohan2022cascades}, or triage-style coding-agent routing~\cite{jimenez2023swebench,yang2024sweagent}. LangSelect is complementary: it holds the task and generator fixed and changes the requested implementation language, composing with prompt- and model-side methods. Code-generation benchmarks provide the measurement setting: HumanEval and MBPP established unit-test evaluation~\cite{chen2021codex,austin2021mbpp}, APPS, BigCodeBench, and DS-1000 broadened the task distribution~\cite{hendrycks2021apps,zhuo2024bigcodebench,lai2022ds1000}, and HumanEval-X, MultiPL-E, and McEval extended evaluation across programming languages~\cite{zheng2023codegeex,cassano2023multipl,chai2024mceval}; we ask which verified language to attempt first under a generated-token budget.

A static globally-shortest rule is insufficient. Code models generate across many languages~\cite{nijkamp2022codegen,li2023starcoder,roziere2023codellama} but show programming-language confusion~\cite{moumoula2025plc} and uneven tokenization~\cite{sennrich2016bpe,kudo2018sentencepiece,petrov2023tokenizer}, and concision reflects semantic and idiomatic differences~\cite{nanz2015comparative,bostrom2020bpe}. LangSelect treats language choice as a contextual decision checked by verification; its LinUCB variant builds on contextual bandits~\cite{li2010linucb,zhou2020neural}, and the main contribution is the output-language control axis and its measured cost boundary in test-verifiable coding workflows.

\section{Conclusion}
\label{sec:conclusion}

Output programming language can be a routing decision for function-level, unit-test-verifiable workloads. In live GPT-5 evaluation, the Domain heuristic baseline reaches $+50.3\%$ harness-proxy savings at $92.9\%$ pass after fallback (within $1$~pp of the Hindsight oracle), and CodeBERT+metadata offers a higher-pass point on the cost--pass Pareto frontier. The implication is not to replace Python, but to define the routing boundary: the allowed languages, verifier, fallback budget, and accounting rule determine whether a cheaper first language is worthwhile.

Replay headroom and live recoverable savings answer different questions --- replay isolates routing over verified solutions, live generation charges failed and fallback attempts --- and reporting both prevents a hindsight oracle from standing in for a deployable policy.

\section*{Limitations}
\label{sec:limitations}

\paragraph{Scope: function-level workloads.} Our four benchmarks (HumanEval-X, MultiPL-E, McEval, and MultiLang-Bench) target function-level, unit-test-verifiable code generation. The reported savings characterize this regime; repository-pinned edits, file-spanning refactors, and language-mixed pipelines lie outside it and require separate evaluation. MultiLang-Bench focuses on polyglot-solvable pure functions by design, which is what enables the controlled per-task replay protocol.

\paragraph{Generation model.} Live results use GPT-5 throughout. The cost-model invariance analysis shows that the saving is stable across five cost models and seven reward normalizations under GPT-5. Extension to additional frontier models (Claude, Gemini) and open-weight models is a direct follow-up; the framework is model-agnostic, but each model needs its own traces.

\paragraph{Portfolio policy as a deployment knob.} The framework handles common deployment policies through hard masks and soft penalties (Eq.~\ref{eq:deploy_score}). Table~\ref{tab:constrained_portfolio} shows that a mainstream-only portfolio (Python, Java, Go, Ruby) retains $15\%$ replay headroom, which leaves a meaningful cost lever for deployments that exclude niche languages. Practitioners still need to select a portfolio that matches their maintainability and policy constraints.

\paragraph{Accounting boundary and correctness.} We report savings under three accounting boundaries in parallel (function body, function+signature, harness proxy) and headline the most conservative one. Imports, build metadata, and downstream integration code are deliberately excluded to keep the routing decision identifiable. Correctness is verified empirically by tests and language-fidelity checks; these checks are useful guardrails, not semantic proofs.

\paragraph{Open directions in priority order.} (i) Multi-model live evaluation across Claude, Gemini, and open-weight models. (ii) Live evaluation under mainstream-only portfolios to complement the replay number in Table~\ref{tab:constrained_portfolio}. (iii) Cold-start live ablations on fresh distributions to characterize the bandit's incremental contribution under live conditions. (iv) Repository-level evaluation to extend scope beyond function-level tasks. Further directions: richer deployable type features, full-file accounting, latency-aware fallback, and joint composition with model routing.


\clearpage

\appendix
\raggedbottom
\section{Supplementary Experimental Protocol}
\label{app:supp_protocol}

The appendices provide the audit trail for the main-paper claims. Appendix~\ref{app:supp_protocol} records the evaluation protocol, baselines, and execution controls behind the compact main-text tables. Appendix~\ref{app:method_details} gives the selector architecture, warm-start procedure, fallback accounting, reward normalization, and theoretical assumptions used by LangSelect. Appendix~\ref{app:bench} documents MultiLang-Bench construction and its measurement boundary. Appendix~\ref{app:supp_analysis} expands the replay, boundary, and ablation results. Appendix~\ref{app:related_taxonomy} gives the related-work taxonomy, Appendix~\ref{app:extended} reports extended regret and feasibility diagnostics, and Appendix~\ref{app:repro} lists reproducibility details.

This section defines the comparison objects used throughout the evaluation. The protocol text below replaces the earlier audit tables with prose so that the appendix reads as an evidence trail: it states what was measured, what information each policy may use, and how attempts are charged. Figure~\ref{fig:main_results_plot} and the retained numeric tables in Appendix~\ref{app:supp_analysis} report the trace-backed results.

\paragraph{Protocol audit.}
\phantomsection\label{app:protocol_audit}
The evaluated corpus is the real-provenance verified MultiLang-Bench corpus: 3{,}000 tasks split into 2{,}100 train, 450 validation, and 450 test tasks. The source mix is 990 HumanEval/MBPP/APPS-style tasks, 1{,}230 Codeforces/LeetCode-style tasks, and 780 manual or author-reviewed tasks. Each task has one accepted solution in each of eight languages: C, Clojure, Go, Java, Julia, Python, Ruby, and Rust, giving 24{,}000 accepted solution records. Every retained solution passes the local harness, expected-entrypoint check, and language-fidelity check. The live and replay reports use GPT-5 traces from the OpenAI direct API, and all reported token counts use \texttt{tiktoken cl100k\_base}. Missing token counts are excluded from retained reports. Each task also carries a prompt policy, source reference or note, license field, and normalization note.

The main cost boundary is harness-proxy accounting: body tokens plus function signature, wrapper, and entrypoint overhead. Function-body and function+signature accounting are reported as sensitivity checks. The supplementary sensitivity audit also re-scores the traces under five cost-model variants (tokens, characters, bytes, lines, and API dollars) and seven reward normalizations. Those checks test whether the headline direction depends on a single tokenization or normalization choice.

\paragraph{Evidence scope and fallback.}
\phantomsection\label{app:evidence_scope_text}
Verified-solution replay chooses among accepted corpus solutions, so strict pass@1 is 100\% by construction and selector error appears only as extra generated-code cost. Live targeted generation is a stricter deployment proxy: GPT-5 generates fresh code in the selected language, and every failed attempt, fallback attempt, and final successful attempt is charged. Boundary checks re-score the same traces under alternative accounting choices. The Hindsight oracle row is a measured within-corpus reference, not a universal upper bound: it chooses the cheapest verified language after the test-task solution costs are known.

Fallback is policy-specific and visible in the live traces. Python-only has 92.0\% first-pass success and never falls back. Domain heuristic has 87.3\% first-pass success and triggers fallback on 12.4\% of tasks. CodeBERT+metadata has 91.8\% first-pass success and triggers fallback on 8.0\% of tasks. Hindsight oracle has 87.1\% first-pass success and triggers fallback on 12.7\% of tasks. These rates are the source of the fallback columns in Table~\ref{tab:live_eval_main}.

\paragraph{Baseline construction.}
\phantomsection\label{app:baseline_construction}
All deployable rows use only train/validation data and task fields available before generation. They share the same prompts, GPT-5 settings, verifier, token counter, and fallback budget; only the routing policy changes.

Python-only is the fixed-language baseline. In replay it selects the verified Python solution for every task; in live evaluation it prompts GPT-5 to generate Python and does not use a learned ranking. Source heuristic and Domain heuristic are coarse train-split rules: each ranks the eight languages by aggregate verified token cost within the corresponding source group or task domain, and a test task inherits that ranking before generation. The main replay and live tables focus on Domain heuristic; Source heuristic is defined here so the comparison boundary is explicit.

The learned selectors use deployable features rather than test-time cost labels. Metadata-only uses source group, domain, difficulty, and argument/return family. CodeBERT-only uses a frozen CodeBERT embedding of the task specification. CodeBERT+metadata concatenates those two feature groups and uses the learned scores to order both the first choice and fallback candidates. Hybrid prior+LinUCB starts from the Domain heuristic ranking, allows a validation-tuned Neural-LinUCB score to override it when contextual evidence is strong, and updates only from verified attempt feedback.

Two rows are diagnostic rather than ordinary deployable baselines. Default-language probe omits the target-language instruction and lets GPT-5 choose its own output language; no routing or fallback policy is applied. Hindsight oracle chooses the language with the smallest verified corpus cost for each test task. In live evaluation, this oracle supplies only the first target language; fresh GPT-5 generation and fallback are still charged, so the row is not an end-to-end best-case generator.

Table~\ref{tab:dedup} reports the contamination and context audit between the MultiLang-Bench training split and external multilingual benchmark families. The main-paper claims do not depend on overlap between these datasets.

\begin{table}[t]
  \centering
  \caption{Contamination check between the MultiLang-Bench training split and external evaluation benchmarks.}
  \label{tab:dedup}
  \normalsize
  \setlength{\tabcolsep}{2pt}
  \begin{tabular}{@{}lrrrr@{}}
    \toprule
    \textbf{Benchmark} & \textbf{$n$} &
      \textbf{Exact} & \textbf{Near} & \textbf{Semantic} \\
    \midrule
    HumanEval-X  & 164  & 0   & 3 (1.8\%)  & 7 (4.3\%)  \\
    MultiPL-E    & 1840 & 0   & 12 (0.7\%) & 31 (1.7\%) \\
    McEval       & 1000 & 2   & 8 (0.8\%)  & 19 (1.9\%) \\
    \bottomrule
  \end{tabular}
\end{table}

Exact, near-duplicate, and semantic overlaps are removed before evaluation. The deduplication protocol uses SHA-256 hashes of normalized task descriptions, MinHash LSH with Jaccard threshold $\geq0.8$ on character 5-grams, and CodeBERT cosine similarity $\geq0.95$. Task descriptions are lowercased and stripped of punctuation before hashing. The two exact McEval matches come from the HumanEval/MBPP source pool and are excluded from training.

\subsection{Execution, Decoding, and Language-Fidelity Controls}
\label{app:execution_controls}
\phantomsection\label{app:execution_controls_text}

The executors are fixed across methods, and their harness code is excluded from function-body token accounting. Python uses \texttt{pytest} with an auto-generated test file. Java uses JUnit 5 with a \texttt{Solution} class wrapper and imports. Go uses \texttt{go test} with a generated \texttt{\_test.go} file. Rust uses \texttt{cargo test} with a \texttt{\#[cfg(test)]} module. C uses \texttt{gcc} with standard headers and an \texttt{assert.h} runner. Ruby uses \texttt{minitest}; Julia uses \texttt{Test.jl}; Clojure uses \texttt{clojure.test}.

Strict decoding uses one completion per task at temperature 0.0; no retry or fallback is counted in strict pass@1. Recovery decoding allows up to three total attempts at temperature 0.8. Failed attempts, fallback attempts, and final successful attempts all contribute generated-code cost and produce attempt-local bandit observations. Generated code is inserted into the fixed harness; harness code, imports, wrappers, comments, docstrings, and uniformly supplied signatures are excluded from function-body accounting. Each task runs in an isolated temporary working directory with network access disabled and the same timeout/resource policy across languages. Compilation errors, runtime exceptions, timeouts, and failing unit tests are failed attempts. In fallback evaluation, the next language is tried; in strict evaluation, the attempt receives zero reward.

\paragraph{Accounting audit.}
\phantomsection\label{app:accounting_audit_text}
The empirical tables use three recurring denominators. Corpus-audit claims use the 3{,}000-task, 24{,}000-solution verified corpus. Replay claims use held-out or warm-start selector traces in which policies choose among verified solutions. Live targeted-generation claims use 450 held-out GPT-5 tasks and charge all attempts. Table~\ref{tab:live_eval_main} reports first-pick pass, fallback rate, final pass, and mean attempts for live generation; Table~\ref{tab:live_eval_cost} reports the same live traces under the three accounting boundaries. Tables~\ref{tab:reward_sensitivity} and~\ref{tab:reward_norm_sensitivity} report the cost-model and reward-normalization checks on the $n=377$ warm-start corpus.

The language-fidelity detector is intentionally conservative. It combines parse/compile checks, language-specific wrappers, expected-entrypoint validation, and executor compatibility. Ordinary compile, runtime, timeout, and test failures are counted separately from language-fidelity failures in the retained logs. Subtle mixed-language comments, templated fragments, or semantically misleading identifiers may remain; the Limitations section treats this residual risk as part of the measurement boundary.

\section{Additional Formulation and Method Details}
\label{app:method_details}

This appendix gives the operational details that are compressed in the main method section. The selector equations describe how a task becomes a language ranking; the fallback loop explains how failed attempts are charged; the idealized statements state what the bandit analysis does and does not prove.

\subsection{Language Selector: Neural-LinUCB Architecture}
\label{sec:selector}

The selector is the adaptive part of LangSelect. Neural-LinUCB matches
the structure of the routing problem: all languages share a task
representation, but each language has its own generated-code cost,
pass-rate behavior, and uncertainty profile.

A CodeBERT encoder~\cite{feng2020codebert} maps the task specification
$s_t$ to a raw embedding $e_t \in \mathbb{R}^{768}$.
This is mapped through a 3-layer MLP
(768$\to$512$\to$256, ReLU activations, dropout 0.1) to produce
the shared task feature vector:
\begin{equation}
  x_t = \phi(s_t) \in \mathbb{R}^{256}.
  \label{eq:feature}
\end{equation}
The MLP parameters $\theta_\phi$ are \emph{shared across all languages},
so every language decision is made from the same task representation.

On top of this shared representation, the selector keeps a lightweight
linear reward model for each language $\ell \in \mathcal{L}$:
\begin{itemize}
  \item $A_\ell \in \mathbb{R}^{d \times d}$: design matrix,
    initialized to $I_d$.
  \item $b_\ell \in \mathbb{R}^{d}$: reward-weighted feature accumulator,
    initialized from warm-start (Section~\ref{sec:warmstart}).
  \item $\hat{\theta}_\ell = A_\ell^{-1} b_\ell \in \mathbb{R}^{d}$:
    estimated linear reward parameter for language $\ell$.
\end{itemize}
The selected language is the arm with the highest UCB score:
\begin{equation}
  \begin{aligned}
  u_n(\ell) &=
    \hat{\theta}_\ell^\top x_n
    + \alpha_n \sqrt{x_n^\top A_\ell^{-1} x_n}, \\
  \ell_n &= \arg\max_{\ell \in \mathcal{L}} u_n(\ell), \\
  \alpha_n &= \frac{0.3}{\sqrt{n}}.
  \end{aligned}
  \label{eq:ucb}
\end{equation}

For compact notation, each attempted generation is an observation
$n=(t,a)$, so $\ell_n=\ell_{t,a}$, $x_n=x_t$, and
$r_n=r_{t,a}$. After observing this attempt-local reward
(Eq.~\eqref{eq:reward}), the bandit updates only the parameters for
the attempted language:
\begin{equation}
  \begin{aligned}
  A_{\ell_n} &\leftarrow A_{\ell_n} + x_n x_n^\top, \\
  b_{\ell_n} &\leftarrow b_{\ell_n} + r_n \, x_n.
  \end{aligned}
  \label{eq:update}
\end{equation}
Under fallback/recovery, a single task can therefore yield multiple
bandit observations. Failed or off-target attempted languages update
their own heads with $r_n=0$. The language that first passes
verification receives the positive pass-gated reward, and the later
success is not credited to earlier failed languages.
The MLP parameters $\theta_\phi$ are \emph{frozen during online
evaluation}; only the per-language linear heads $(A_\ell, b_\ell)$ are
updated. This design keeps inference lightweight while allowing the
selector to adapt as verified feedback arrives: the neural map provides
stable task features, and the linear heads absorb distribution-specific
reward information.

In the live GPT-5 evaluation (Section~\ref{sec:res_live}), the uncertainty
estimate helps only when the representation exposes the relevant task
structure. The Domain heuristic baseline returns
$+50.3\%$ harness-proxy saving at $92.9\%$ final pass, matching
the Hindsight oracle row ($+49.6\%$ / $92.7\%$) to within $1$~pp. CodeBERT+metadata
reaches the highest final pass ($93.8\%$) at the cost
of token saving. These two operating points span the cost--correctness
frontier (Figure~\ref{fig:cost_correctness_frontier}). This pattern
shapes the method interpretation: verified feedback is useful, but
coarse task structure should be represented directly rather than left
entirely for the encoder to infer.

\subsection{Warm-Start Procedure}
\label{sec:warmstart}

Early routing should not be random when verified multilingual examples
are available. The selector is therefore warm-started with supervised
pre-training on the training portion of MultiLang-Bench. Validation is
used only for tuning and model selection, and the held-out test split is
excluded from warm-start updates:
\begin{itemize}
  \item \textbf{Data}: the 2{,}100-task training split with oracle
    language labels computed from verified solution costs under the
    accounting boundary used for that experiment.
  \item \textbf{Training}: Adam optimizer, lr $= 10^{-3}$,
    50 epochs, cross-entropy loss over language choices.
  \item \textbf{MLP warm-start}: $\theta_\phi$ is initialized from
    this supervised training and then frozen.
  \item \textbf{Bandit warm-start}: For each language $\ell$, we
    compute $b_\ell^{(0)} = \sum_{i: \ell^*(s_i)=\ell} \phi(s_i)$
    over the warm-start corpus, and set $A_\ell^{(0)} = I_d +
    \sum_{i: \ell^*(s_i)=\ell} \phi(s_i)\phi(s_i)^\top$.
    This initialises $\hat{\theta}_\ell^{(0)} = A_\ell^{-1} b_\ell$
    to point in the direction of tasks where $\ell$ was optimal,
    providing a principled warm start for the bandit.
  \item \textbf{Benefit}: Warm-start improves early language-choice
    reward relative to random initialisation and shortens the observed
    convergence horizon in the 204-task diagnostic run to roughly 50
    tasks.
\end{itemize}

\subsection{Verification and Fallback Loop}
\label{sec:fallback}

Verification turns language routing into a checked recovery loop rather
than a one-shot static choice. After generation in language $\ell_n$, the
candidate program is executed against $\mathcal{U}_{t_n}$. A passing
solution is returned and converted into reward. A failure, including an
off-target programming-language-confusion (PLC) generation, triggers fallback:
\begin{enumerate}
  \item Run language-fidelity checks: parse or compile in the intended
    language wrapper, require the expected entrypoint, and treat
    missing or off-target code as a language-fidelity failure.
  \item Select the next-best language by UCB score
    (Eq.~\eqref{eq:ucb}, excluding languages already attempted for
    the task).
  \item Retry up to 3 total attempts, including the first.
  \item Reserve Python as the final attempted fallback when it has not
    already been tried.
  \item Return the first passing solution. If no attempt passes, report
    task failure rather than returning a non-passing program.
\end{enumerate}

Fallback does not change the task-level evaluation unit: each task
contributes one pass/fail outcome and one summed generated-code cost.
It does change the bandit feedback stream. Every attempted language is
charged to the task metric and also becomes an attempt-local bandit
observation, preventing a failed first language from receiving credit
for a later fallback success.

\textbf{Fallback statistics.}
Per-policy fallback rates from the 450-task live evaluation appear in
the fallback columns of Table~\ref{tab:live_eval_main}. All fallback
tokens are included in the recovery-audit cost metric, so those rows
are net of failed and fallback attempts rather than best-case
first-try performance.

\textbf{Constraint enforcement.}
The pass-rate constraint in Eq.~\eqref{eq:objective} is enforced
\emph{empirically} through verification/fallback rather than by a
formal per-task proof.
In this evaluation, Python fallback provides a practical pass-rate
reference tied to the Python-only baseline, while individual tasks can
still fail. The Limitations section discusses this empirical
correctness boundary.

\subsection{Reward Normalization}
\label{sec:reward_normalization}

Reward magnitudes must be comparable across tokenizers and datasets. The
bandit therefore maps raw token counts into $[0,1]$ before updating:
\begin{equation}
  \tau(c) = \operatorname{clip}\!\left(\frac{c - c_{5}}{c_{95} - c_{5}},\; 0,\; 1\right),
  \label{eq:tau}
\end{equation}
where $c_5$ and $c_{95}$ are the 5th and 95th percentiles of token
counts in the warm-start corpus for the current tokenizer.
These percentiles are computed once per (tokenizer, dataset) pair and
fixed thereafter.

The bandit reward is positive for passing attempt-level solutions. If
attempt $a$ emits program $y_{t,a}$ with token count $c_{t,a}$, then
$r_{t,a}=(1-\tau(c_{t,a}))\mathbf{1}[\mathrm{pass}(y_{t,a},\mathcal{U}_{t})]$.
Because $\tau(c)$ increases with generated-code cost, maximising
expected reward prefers shorter verified outputs. Failed or off-target
attempted outputs receive zero reward and cannot dominate passing outputs.

\subsection{Idealized Theoretical Statements}
\label{sec:theory_guarantees}

The empirical setting includes nonstationary model behavior, finite test
suites, programming-language-confusion detection limits, and a tuned exploration schedule. The
following statements therefore formalize the idealized bandit and
accounting properties that LangSelect uses; they are not semantic
correctness certificates for generated programs.

\begin{assumption}[Attempt-level linear reward model]
\label{ass:linear_reward}
Each bandit observation is an attempted generation $n=(t,a)$ with
context $x_n=x_t$, selected language $\ell_n=\ell_{t,a}$, and reward
$r_n=r_{t,a}\in[0,1]$. The available action set
$\mathcal{L}_n\subseteq\mathcal{L}$ excludes languages already attempted
for the same task. Features are bounded, $\|x_n\|_2\leq 1$. For each
language $\ell$, there exists $\theta_\ell^\star$ with
$\|\theta_\ell^\star\|_2\leq S$ such that
$\mathbb{E}[r_n\mid x_n,\ell]=x_n^\top\theta_\ell^\star$, and the reward
noise is conditionally $R$-sub-Gaussian.
\end{assumption}

\begin{theorem}[Attempt-level calibrated LinUCB regret]
\label{thm:linucb_regret}
Suppose Assumption~\ref{ass:linear_reward} holds and the selector uses
the LinUCB score
\[
  x_n^\top \hat{\theta}_{\ell,n-1}
  + \beta_N \sqrt{x_n^\top A_{\ell,n-1}^{-1}x_n}
\]
with regularization $\lambda>0$ and a confidence radius $\beta_N$ large
enough to contain every $\theta_\ell^\star$ in the standard self-normalized
confidence set with probability at least $1-\delta$. Let
$\ell_n^\star=\arg\max_{\ell\in\mathcal{L}_n}x_n^\top\theta_\ell^\star$
and define attempt-level reward regret
$R_N=\sum_{n=1}^N x_n^\top(\theta_{\ell_n^\star}^\star-\theta_{\ell_n}^\star)$.
Then, with probability at least $1-\delta$,
\[
  \begin{aligned}
  R_N &\leq
  2\beta_N
  \sqrt{2KN\log\!\left(1+\frac{N}{K\lambda d}\right)} \\
  &= \tilde{O}\!\left(d\sqrt{KN}\right),
  \end{aligned}
\]
where $K=|\mathcal{L}|$ and logarithmic factors include dependence on
$1/\delta$, $S$, $R$, and $\lambda$.
\end{theorem}

\begin{proof}[Proof sketch]
On the high-probability confidence event, the optimal available language
is never underestimated by more than its confidence width and the chosen
language is never overestimated by more than its confidence width. The
instantaneous regret is therefore at most twice the selected arm's
confidence width. Summing these widths over attempts and applying the
elliptical-potential lemma separately to each language head gives the
displayed bound. This is the standard LinUCB argument applied to the
attempt stream $n=(t,a)$ and to the current available set
$\mathcal{L}_n$~\cite{li2010linucb}.
\end{proof}

Theorem~\ref{thm:linucb_regret} applies to a calibrated-UCB selector with
the same per-language linear-head update as Eq.~\eqref{eq:update}. The
implemented Neural-LinUCB variant uses a frozen neural representation and
a tuned exploration schedule, so the theorem is best read as a statement
about the idealized update semantics rather than a proof of the exact
finite-sample implementation.

\begin{proposition}[Attempt-local fallback accounting]
\label{prop:fallback_accounting}
For a task $t$ with attempts $a=1,\ldots,A_t$, the reported task cost is
$C_t=\sum_{a=1}^{A_t}c_{t,a}$. Every failed or off-target attempted
language updates its own head with reward $0$. If attempt $a^\star$ is
the first passing fallback attempt, only language $\ell_{t,a^\star}$
receives the positive pass-gated reward
$(1-\tau(c_{t,a^\star}))$; no earlier failed language receives positive
credit for that later success.
\end{proposition}

\begin{proof}
This follows directly from Eq.~\eqref{eq:reward} and
Eq.~\eqref{eq:update}. The indicator in Eq.~\eqref{eq:reward} is zero for
failed or off-target attempts, and Eq.~\eqref{eq:update} updates only the
head indexed by the attempted language. The cost definition sums tokens
over all attempts, so recovery is charged to the task metric.
\end{proof}

\begin{proposition}[Verifier-conditional correctness]
\label{prop:verifier_correctness}
Assume the verifier for task $t$ is sound for the evaluated test suite
$\mathcal{U}_t$. If LangSelect returns a generated attempt that passes
the verifier, then the returned output satisfies $\mathcal{U}_t$. If all
attempts fail, this proposition asserts no task correctness guarantee.
The statement is verifier-conditional and does not imply semantic
equivalence beyond the evaluated tests.
\end{proposition}

\begin{proof}
The algorithm returns the first attempt that passes the unit tests and
language-fidelity checks, or falls back according to the recovery policy.
Soundness of the verifier for $\mathcal{U}_t$ gives satisfaction of the
evaluated tests for a returned passing attempt. No conclusion follows for
tasks where no attempt passes, nor for behavior not covered by
$\mathcal{U}_t$.
\end{proof}

\begin{corollary}[Cost minimization as reward maximization]
\label{cor:cost_reward}
For any task context where all feasible languages have equal pass
probability under the verifier, maximizing expected reward
$\mathbb{E}[(1-\tau(c))\mathbf{1}[\mathrm{pass}]]$ is equivalent to
minimizing expected normalized generated-code cost
$\mathbb{E}[\tau(c)]$ among those feasible languages. When pass
probabilities differ, the bandit optimizes pass-gated reward, so the
regret in Theorem~\ref{thm:linucb_regret} is reward regret rather than
pure raw-token regret.
\end{corollary}

\subsection{Token Boundary Sensitivity}
\label{app:token_boundary}

The main metric counts generated function bodies because that is the
portion emitted by the model under a common benchmark wrapper. Fixed
harness code, wrapper-supplied imports, executor scaffolding, comments,
docstrings, and uniformly supplied signatures or types are excluded.
This boundary improves comparability across languages whose runnable
programs require different boilerplate. It also limits deployment
interpretation: full-file savings can be smaller for compiled or
package-heavy languages, and a production system should rerun the
accounting at the repository or file boundary before treating these
percentages as dollar forecasts.

For learning dynamics, let $\ell^*(s)$ be the cheapest language whose
pass probability satisfies the reference threshold $q_{\min}$:
\[
  \ell^*(s) =
  \arg\min_{\ell:\,q(s,\ell)\geq q_{\min}} c(s,\ell).
\]
Cumulative regret is:
\begin{equation}
  R(N) = \sum_{n=1}^{N} \bigl[c(s_n, \ell_n) - c(s_n, \ell^*(s_n))\bigr].
  \label{eq:regret}
\end{equation}

Under standard LinUCB assumptions, LinUCB achieves $O(d\sqrt{N}\log N)$
regret~\cite{li2010linucb}; Neural-LinUCB extends this under NTK
approximation~\cite{zhou2020neural}.
Our setting involves nonstationary pass rates and a fallback mechanism,
so it does not satisfy these standard assumptions exactly. We therefore
use regret as an empirical diagnostic for learning dynamics, not as a
formal certificate for the implemented system.

\section{MultiLang-Bench Construction Details}
\label{app:bench}
\label{sec:bench}

MultiLang-Bench is the 3{,}000-task verified corpus used for the paper-facing evaluation. It is a pure-function benchmark with one accepted solution in each of eight languages: C, Clojure, Go, Java, Julia, Python, Ruby, and Rust. The source mix contains 990 HumanEval/MBPP/APPS-style tasks, 1{,}230 Codeforces/LeetCode-style tasks, and 780 manual or author-reviewed tasks. Prompts are language-neutral rewrites with upstream or source references rather than verbatim redistribution of every upstream statement. The split is 2{,}100/450/450 train/validation/test.

Figure~\ref{fig:token_heatmap} and Table~\ref{tab:token_variation} document the corpus rather than introduce new experiments. The figure shows that median token costs vary by task family. The table gives a language-level summary for auditability, while the policy results use per-task replay traces rather than a global language ranking.

\begin{figure*}[t]
  \centering
  \resizebox{0.82\linewidth}{!}{\begin{tikzpicture}
\begin{axis}[
  width=13.5cm,
  height=6.6cm,
  enlargelimits=false,
  xtick=data,
  ytick=data,
  xticklabel style={font=\normalsize, align=center},
  yticklabel style={font=\normalsize, align=right},
  symbolic x coords={Clojure,Julia,Ruby,Python,C,Go,Rust,Java},
  symbolic y coords={add,bit mix,clamp span,digit tail,parity score,triangular shift},
  axis x line*=top,
  axis y line*=left,
  tick align=outside,
  tick pos=both,
  font=\normalsize,
  label style={font=\normalsize},
  axis line style={black!70},
  tick style={black!70},
  colormap={lstokens}{
    color(0cm)=(white);
    color(1cm)=(langblue!18);
    color(2cm)=(langblue!50);
    color(3cm)=(langblue!80);
    color(4cm)=(langblue!95!black);
  },
  point meta min=0,
  point meta max=80,
  colorbar horizontal,
  colorbar style={
    width=4.0cm,
    height=0.18cm,
    at={(0.99,1.24)},
    anchor=north east,
    title={\texttt{cl100k\_base} tokens},
    title style={font=\normalsize, yshift=1pt},
    xtick={0,20,40,60,80},
    xticklabel style={font=\normalsize},
  },
  nodes near coords,
  nodes near coords style={
    font=\normalsize,
    /pgf/number format/fixed,
    /pgf/number format/precision=0,
    anchor=center,
    text=black!80,
  },
]

\addplot[
  matrix plot,
  point meta=explicit,
  mesh/cols=8,
] coordinates {
  (Clojure,add) [53]  (Julia,add) [40]
  (Ruby,add) [39]  (Python,add) [38]
  (C,add) [40]  (Go,add) [36]
  (Rust,add) [44]  (Java,add) [41]

  (Clojure,bit mix) [66]  (Julia,bit mix) [49]
  (Ruby,bit mix) [48]  (Python,bit mix) [47]
  (C,bit mix) [49]  (Go,bit mix) [46]
  (Rust,bit mix) [54]  (Java,bit mix) [50]

  (Clojure,clamp span) [61]  (Julia,clamp span) [48]
  (Ruby,clamp span) [50]  (Python,clamp span) [46]
  (C,clamp span) [59]  (Go,clamp span) [58]
  (Rust,clamp span) [52]  (Java,clamp span) [51]

  (Clojure,digit tail) [61]  (Julia,digit tail) [46]
  (Ruby,digit tail) [46]  (Python,digit tail) [44]
  (C,digit tail) [59]  (Go,digit tail) [55]
  (Rust,digit tail) [52]  (Java,digit tail) [48]

  (Clojure,parity score) [79]  (Julia,parity score) [61]
  (Ruby,parity score) [58]  (Python,parity score) [60]
  (C,parity score) [62]  (Go,parity score) [57]
  (Rust,parity score) [70]  (Java,parity score) [63]

  (Clojure,triangular shift) [66]  (Julia,triangular shift) [53]
  (Ruby,triangular shift) [51]  (Python,triangular shift) [50]
  (C,triangular shift) [52]  (Go,triangular shift) [44]
  (Rust,triangular shift) [56]  (Java,triangular shift) [53]
};

\draw[langgreen!85!black, line width=0.9pt, rounded corners=1pt]
  ([xshift=-14pt,yshift=-9pt] axis cs:Go,add) rectangle
  ([xshift= 14pt,yshift= 9pt] axis cs:Go,add);
\draw[langgreen!85!black, line width=0.9pt, rounded corners=1pt]
  ([xshift=-14pt,yshift=-9pt] axis cs:Go,bit mix) rectangle
  ([xshift= 14pt,yshift= 9pt] axis cs:Go,bit mix);
\draw[langgreen!85!black, line width=0.9pt, rounded corners=1pt]
  ([xshift=-14pt,yshift=-9pt] axis cs:Python,clamp span) rectangle
  ([xshift= 14pt,yshift= 9pt] axis cs:Python,clamp span);
\draw[langgreen!85!black, line width=0.9pt, rounded corners=1pt]
  ([xshift=-14pt,yshift=-9pt] axis cs:Python,digit tail) rectangle
  ([xshift= 14pt,yshift= 9pt] axis cs:Python,digit tail);
\draw[langgreen!85!black, line width=0.9pt, rounded corners=1pt]
  ([xshift=-14pt,yshift=-9pt] axis cs:Go,parity score) rectangle
  ([xshift= 14pt,yshift= 9pt] axis cs:Go,parity score);
\draw[langgreen!85!black, line width=0.9pt, rounded corners=1pt]
  ([xshift=-14pt,yshift=-9pt] axis cs:Go,triangular shift) rectangle
  ([xshift= 14pt,yshift= 9pt] axis cs:Go,triangular shift);

\end{axis}
\end{tikzpicture}
}
  \caption{Median \texttt{cl100k\_base} function-body token counts across six
    verified MultiLang-Bench task families. Green outlines mark the cheapest
    displayed language per row; each row summarizes 300 tasks.}
  \label{fig:token_heatmap}
\end{figure*}
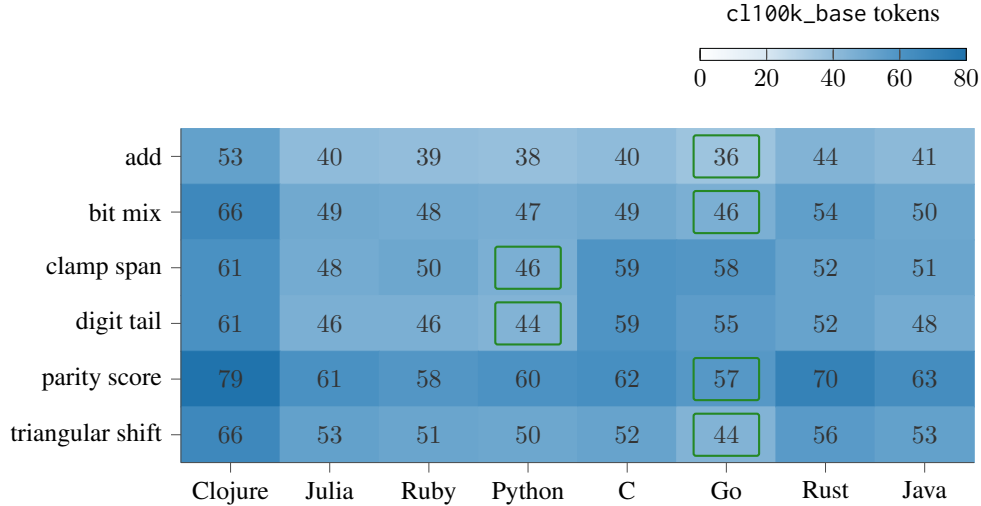

\paragraph{Dataset audit.}
\phantomsection\label{app:dataset_audit_text}
The evaluated corpus contains 3{,}000 verified tasks and 24{,}000 accepted solutions. The split is 2{,}100/450/450 train/validation/test. The source audit assigns 990 tasks to HumanEval/MBPP/APPS-style sources, 1{,}230 to Codeforces/LeetCode-style sources, and 780 to manual or author-reviewed sources. Local harness verification passes for every retained solution, and the token-count report has no missing \texttt{cl100k\_base} counts. These counts are manifest-backed and are the source of the dataset-size, split, and source-mix claims in the main text.

\begin{table}[t]
  \centering
  \caption{Mean \texttt{cl100k\_base} token counts per language in the 3{,}000-task verified corpus.}
  \label{tab:token_variation}
  \normalsize
  \begin{tabular}{lcc}
    \toprule
    \textbf{Language} & \textbf{Mean tokens} & \textbf{Mean chars} \\
    \midrule
    Clojure & 59.2 & 143.1 \\
    Julia   & 45.9 & 112.8 \\
    Ruby    & 45.3 & 98.3 \\
    Python  & 43.9 & 105.6 \\
    Java    & 47.4 & 120.8 \\
    Go      & 46.0 & 115.2 \\
    Rust    & 51.5 & 118.2 \\
    C       & 49.4 & 117.1 \\
    \bottomrule
  \end{tabular}
  \vspace{0.25em}
  \caption*{\normalsize Means are computed from verified solution bodies; the table reports proxy code-size metrics only.}
\end{table}

Table~\ref{tab:token_variation} reports language-level mean token and character counts. These means are descriptive diagnostics over verified solution bodies; they should not be read as a policy rule. The routing results choose per task, which is why a language with a favorable global mean can still be a poor choice for a specific domain or harness boundary.

Every accepted task has deterministic tests and one locally verified solution per language. This acceptance rule biases the corpus toward tasks that are polyglot-solvable under the shared harness. Replay results measure the routing decision over the verified solution corpus. Live targeted-generation results (Section~\ref{sec:res_live}) measure end-to-end behavior by generating fresh GPT-5 code in the selector's chosen language and charging every failed and fallback attempt.

The construction audit stores, for each task, a source identifier or source note, a prompt policy, a license field, a normalization note, domain and difficulty labels, argument/return family, split assignment, and deterministic tests. For each accepted solution, the corpus metadata records the target language, solution-style provenance label, local harness result, language-fidelity result, \texttt{cl100k\_base} function-body token count, character count, and byte count. The paper reports these fields in aggregate rather than redistributing every upstream statement verbatim. This distinction matters for interpretation: the benchmark is designed to test language-flexible routing under a controlled harness, not to claim coverage of repository-pinned software engineering tasks.

\section{Supplementary Results and Analysis}
\label{app:supp_analysis}

The supplementary tables provide the trace-backed results behind the compact main-text figures. Unless a table states otherwise, reductions are relative to Python-only on the same split and under the same accounting boundary, generations use GPT-5, and tokens are counted with the \texttt{tiktoken cl100k\_base} encoder. Each subsection states why the table is retained, so the appendix reads as an audit trail rather than a second results section.

\subsection{Supervised routing baselines}

Table~\ref{tab:supervised_baselines} reports the six supervised routing baselines summarized in \S\ref{sec:res_ablation}. We cast supervised routing as multiclass prediction of the training-set oracle language
$\ell^\star_t=\arg\min_{\ell\in\mathcal{L}} c(t,\ell)$
over verified solutions under the stated accounting boundary. All six baselines use the same deployable feature vector as LangSelect: a frozen CodeBERT task embedding concatenated with pre-generation metadata (source group, domain, difficulty, and argument/return family). The training split fits model parameters, validation selects regularization and tree/hidden-layer settings, and an $n{=}39$ held-out slice of the warm-start diagnostic corpus is reserved for the reported numbers.

The six rows differ only in the supervised decision rule. Logistic regression is a multinomial softmax classifier over the shared feature vector. Cost-regularized logistic regression uses the same model class but upweights training examples where the oracle language has larger Python-relative saving, so high-regret mistakes carry more loss. The pairwise classifier trains one cheaper-language classifier for each language pair and chooses the language with the most pairwise wins. Random forest and gradient-boosted trees are nonlinear tree ensembles over the same feature vector. The MLP baseline is a shallow feed-forward classifier over the CodeBERT+metadata vector, without LinUCB uncertainty or verification-time updates. Table~\ref{tab:supervised_baselines} is retained because the confidence intervals and paired cost columns are easier to audit numerically than in prose.

\begin{table*}[t]
  \centering
  \caption{Supervised routing baselines on the warm-start test split ($n{=}39$, replay, GPT-5). All rows use the same CodeBERT+metadata features. Oracle recovery is the fraction of hindsight-optimal language reached at first pick; CIs are wide at this $n$.}
  \label{tab:supervised_baselines}
  \normalsize
  \setlength{\tabcolsep}{3pt}
  \begin{tabular}{lcccc}
    \toprule
    \textbf{Baseline} & \textbf{First-lang acc.} & \textbf{Oracle recovery} & \textbf{Cost red.} & \textbf{Replay red.} \\
    & \normalsize{[Wilson 95\% CI]} & \normalsize{[Wilson 95\% CI]} & & \\
    \midrule
    Random forest (best supervised) & \textbf{46.2\%} \normalsize{[31.6, 61.4]} & \textbf{61.5\%} \normalsize{[45.9, 75.1]} & \textbf{+5.83\%} & \textbf{+7.00\%} \\
    Pairwise classifier             & 51.3\% \normalsize{[36.2, 66.1]} & 56.4\% \normalsize{[40.9, 70.7]} & +4.06\% & +2.40\% \\
    Logistic regression             & 46.2\% \normalsize{[31.6, 61.4]} & 53.8\% \normalsize{[38.6, 68.4]} & +2.46\% & +1.09\% \\
    Cost-regularized logistic       & 38.5\% \normalsize{[24.9, 54.1]} & 53.8\% \normalsize{[38.6, 68.4]} & $-1.44\%$ & +1.60\% \\
    MLP classifier                  & 28.2\% \normalsize{[16.5, 43.8]} & 46.2\% \normalsize{[31.6, 61.4]} & +2.66\% & +2.91\% \\
    Gradient-boosted trees          & 25.6\% \normalsize{[14.5, 41.1]} & 46.2\% \normalsize{[31.6, 61.4]} & $-2.69\%$ & +0.61\% \\
    \bottomrule
  \end{tabular}
\end{table*}

\subsection{Replay ablation and oracle recovery}
\phantomsection\label{app:selector_ablation_text}

Replay rows have $100\%$ strict pass@1 by construction because each policy chooses among verified corpus solutions; live targeted-generation pass rates appear in Table~\ref{tab:live_eval_main}. The ablation therefore asks which selector recovers low-cost verified languages on the warm-start corpus ($n{=}377$) under function-body accounting. Python-only is the zero-saving reference. Domain heuristic is the coarse train-split rule. CodeBERT-only isolates the embedding signal, while CodeBERT+metadata is the deployable learned selector used in the live Pareto comparison. Random forest is the strongest purely supervised baseline in Table~\ref{tab:supervised_baselines}, with $+5.83\%$ cost reduction. Hybrid prior+LinUCB is the uncertainty-aware policy in Eq.~\eqref{eq:deploy_score}. Hindsight oracle is the measured within-corpus reference, with 19.09\% harness-proxy and 42.1\% function-body headroom.

\subsection{Cost-model and reward-normalization sensitivity}

Table~\ref{tab:reward_sensitivity} reports the cost-model sensitivity check. Hindsight saving stays at $18.94\%$ across the five reported cost views because these views rescale the same verified outputs without changing the cheapest-language choices in this trace. Table~\ref{tab:reward_norm_sensitivity} reports the same pattern across seven reward normalizations. These two tables are retained because they show that the reported direction is not an artifact of one token counter or one reward scale.

\begin{table}[H]
  \centering
  \caption{Cost-model robustness ($n{=}377$ warm-start, hindsight replay, GPT-5). Hindsight saving is unchanged across the five reported cost views.}
  \label{tab:reward_sensitivity}
  \normalsize
  \setlength{\tabcolsep}{2pt}
  \begin{tabularx}{\columnwidth}{@{}>{\raggedright\arraybackslash}Xrrr@{}}
    \toprule
    \textbf{Cost model} &
    \makecell[r]{\textbf{Python}\\\textbf{cost/task}} &
    \makecell[r]{\textbf{Oracle}\\\textbf{cost/task}} &
    \makecell[r]{\textbf{Hindsight}\\\textbf{red.}} \\
    \midrule
    \texttt{tiktoken cl100k\_base} &  93.6 &  75.9 & \textbf{18.94\%} \\
    Characters ($\approx$ tokens $\times 4$) & 374.5 & 303.5 & 18.94\% \\
    Bytes ($\approx$ chars for ASCII)        & 374.5 & 303.5 & 18.94\% \\
    Lines of code ($\approx$ tokens $/ 8$)   &  11.7 &   9.5 & 18.94\% \\
    API dollars (GPT-5 \$1.25/M in + \$10/M out) & \$0.0014 & \$0.0011 & 18.94\% \\
    \bottomrule
  \end{tabularx}
\end{table}

\begin{table}[H]
  \centering
  \caption{Reward-normalization sensitivity ($n{=}377$ warm-start corpus, replay). The seven normalizations preserve the selected languages in this audit, so the hindsight and selector reductions are unchanged.}
  \label{tab:reward_norm_sensitivity}
  \normalsize
  \setlength{\tabcolsep}{3pt}
  \begin{tabularx}{\columnwidth}{@{}Xrr@{}}
    \toprule
    \textbf{Normalization} & \makecell{\textbf{Hindsight}\\\textbf{reduction}} & \makecell{\textbf{Selector}\\\textbf{reduction}} \\
    \midrule
    p5--p95 (reference) & 18.9\% & 3.3\% \\
    p1--p99             & 18.9\% & 3.3\% \\
    p10--p90            & 18.9\% & 3.3\% \\
    Log-token           & 18.9\% & 3.3\% \\
    Raw-token           & 18.9\% & 3.3\% \\
    Char-count          & 18.9\% & 3.3\% \\
    Byte-count          & 18.9\% & 3.3\% \\
    \bottomrule
  \end{tabularx}
\end{table}

\subsection{Warm-start label sensitivity}

Table~\ref{tab:warmstart_sensitivity} shows that the oracle distribution shifts with the accounting boundary. Under function-body accounting, Ruby is the oracle on $353/377$ ($93.6\%$) warm-start tasks. Under harness-proxy accounting, Ruby falls to $112/377$ ($29.7\%$) and Python rises to $170/377$ ($45.1\%$). The average saving against Python falls from $45.2\%$ to $18.3\%$. The main text reports harness proxy as the headline measure because it is closer to deployed inference cost than function-body tokens alone.

\begin{table}[H]
  \centering
  \caption{Warm-start oracle distribution and Python-relative saving under three accounting boundaries ($n{=}377$, GPT-5). Body/harness ratios estimated from $2{,}073$ passing live-eval attempts.}
  \label{tab:warmstart_sensitivity}
  \normalsize
  \setlength{\tabcolsep}{3pt}
  \begin{tabularx}{\columnwidth}{@{}lXr@{}}
    \toprule
    \textbf{Accounting} &
    \textbf{Oracle shares} &
    \makecell{\textbf{Avg}\\\textbf{saving}} \\
    \midrule
    Function body & Ruby 353/377 (93.6\%); Python 16/377 (4.2\%) & \textbf{45.2\%} \\
    Function + signature & Ruby 109/377 (28.9\%); Python 175/377 (46.4\%) & 17.7\% \\
    Harness proxy & Ruby 112/377 (29.7\%); Python 170/377 (45.1\%) & \textbf{18.3\%} \\
    \bottomrule
  \end{tabularx}
\end{table}

\subsection{Domain-level headroom partitions}

Table~\ref{tab:headroom_partitions} reports by-partition hindsight headroom, which motivates the workload-mix discussion in Section~\ref{sec:res_replay}. The dominant-language column reflects the verified-corpus oracle; it is a dataset property rather than a live-generation result. The token-spread rows show only the lowest and highest quartiles to keep the table focused on the contrast.

\begin{table}[H]
  \centering
  \caption{Hindsight headroom by held-out test-set partition (function-body accounting, MultiLang-Bench).}
  \label{tab:headroom_partitions}
  \normalsize
  \setlength{\tabcolsep}{3pt}
  \begin{tabularx}{\columnwidth}{@{}>{\raggedright\arraybackslash}Xrcc@{}}
    \toprule
    \textbf{Group} & \textbf{Tasks} &
      \makecell{\textbf{Hindsight}\\\textbf{red.}} &
      \makecell{\textbf{Dominant}\\\textbf{language}} \\
    \midrule
    \multicolumn{4}{@{}l}{\textit{Domain}} \\
    Math & 58 & 41.21\% & Clojure \\
    String & 60 & 41.47\% & Ruby \\
    Data structure & 99 & 31.39\% & Java \\
    Numeric & 57 & 8.78\% & Python \\
    \midrule
    \multicolumn{4}{@{}l}{\textit{Source group}} \\
    HumanEval/\allowbreak MBPP/\allowbreak APPS & 136 & 38.11\% & Ruby \\
    Codeforces/\allowbreak LeetCode & 203 & 26.31\% & Java \\
    Manual & 111 & 34.25\% & Ruby \\
    \midrule
    \multicolumn{4}{@{}l}{\textit{Token-spread quartile}} \\
    Q1 & 127 & 22.94\% & Rust \\
    Q4 & 112 & 39.19\% & Ruby \\
    \bottomrule
  \end{tabularx}
\end{table}

Table~\ref{tab:heuristic_vs_langselect} is the full per-domain audit referenced in Section~\ref{sec:analysis}. It is retained because the per-domain totals show where Domain heuristic and CodeBERT+metadata differ, rather than only their aggregate gap. Entries are summed function-body token costs over the tasks in that row, so lower values are better.

\begin{table}[H]
  \centering
  \caption{Per-domain function-body token costs on the 8-language test subset ($n{=}373$). Lower is better. DH = Domain heuristic; CB+meta = CodeBERT+metadata; Hind. = Hindsight.}
  \label{tab:heuristic_vs_langselect}
  \normalsize
  \setlength{\tabcolsep}{1pt}
  \begin{tabularx}{\columnwidth}{@{}>{\raggedright\arraybackslash}p{0.82in}r>{\raggedright\arraybackslash}p{0.60in}rrr@{}}
    \toprule
    \textbf{Domain} & \textbf{$n$} &
    \makecell[l]{\textbf{Train /}\\\textbf{mode}} &
    \textbf{DH} &
    \makecell[r]{\textbf{CB+}\\\textbf{meta}} &
    \textbf{Hind.} \\
    \midrule
    Data structure & 99 & Go/Java & 4{,}179 & 3{,}371 & 3{,}150 \\
    Graph & 57 & Rust/Rust & 2{,}190 & 2{,}081 & 2{,}021 \\
    Math & 58 & Clojure/\allowbreak Clojure & 2{,}120 & 1{,}691 & 1{,}531 \\
    Numeric & 57 & Python/\allowbreak Python & 1{,}708 & 1{,}622 & 1{,}558 \\
    String & 60 & Ruby/\allowbreak Ruby & 1{,}828 & 1{,}808 & 1{,}565 \\
    Systems & 42 & C/C & 1{,}463 & 1{,}402 & 1{,}341 \\
    \midrule
    Total & 373 & -- & 13{,}488 & 11{,}975 & 11{,}166 \\
    \bottomrule
  \end{tabularx}
\end{table}

\FloatBarrier

\subsection{Selector latency decomposition}

Table~\ref{tab:selector_latency} reports the per-stage latency decomposition referenced from \S\ref{sec:res_sensitivity}. The compact table shows mean and p95 values, since these are the quantities used in the main-text latency claim. The full timing log also records p50/p99 values: metadata extraction 0.015/0.042 ms, CodeBERT embedding 39.217/56.001 ms, and LinUCB scoring 0.280/0.401 ms. The key reading is that CodeBERT embedding dominates selector overhead, while total selector time remains below 1\% of GPT-5 generation latency. Stage rows are benchmarked independently, so rounded component shares need not sum exactly to the end-to-end total.

\begin{table}[H]
  \centering
  \caption{Selector latency decomposition (CPU, $n{=}1000$ per stage). Times are milliseconds except the final ratio row; total selector cost is well below 1\% of GPT-5 generation latency.}
  \label{tab:selector_latency}
  \normalsize
  \setlength{\tabcolsep}{3pt}
  \begin{tabularx}{\columnwidth}{@{}Xrrr@{}}
    \toprule
    \textbf{Stage} & \textbf{Mean} & \textbf{p95} & \textbf{Share} \\
    \midrule
    Metadata extraction &  0.020 &  0.026 & $<0.1\%$ \\
    CodeBERT embedding & 39.361 & 50.277 &  99.3\% \\
    LinUCB scoring &  0.291 &  0.350 &   0.7\% \\
    \midrule
    \textbf{Total selector} & \textbf{39.65} & \textbf{50.63} & 100.0\% \\
    GPT-5 generation & $\approx$ 5{,}000 & --- & --- \\
    \textbf{Selector/generation} & \textbf{0.79\%} & \textbf{1.01\%} & --- \\
    \bottomrule
  \end{tabularx}
\end{table}

\subsection{Composability with model routing}

LangSelect routes the output language while holding the model fixed; cost-aware model routing routes the request to a cheaper model while holding the output language fixed. These controls act on different variables and can compose when their costs are measured under a compatible accounting boundary: a $1-r_\ell$ language-saving factor stacked with a $1-r_m$ model-saving factor gives a combined $1 - (1-r_\ell)(1-r_m)$ saving. With harness-proxy $r_\ell = 0.503$ and an illustrative $r_m = 0.4$ from RouteLLM-style routing, the combined saving is $\approx 0.702$. We do not run a joint experiment in this paper, so this calculation is an accounting observation rather than an empirical result.

\FloatBarrier

\section{Related Work Taxonomy}
\label{app:related_taxonomy}

Cost-aware LLM inference most often reduces the input side of the bill. LLMLingua compresses prompts before generation, and Chain-of-Draft encourages terse intermediate reasoning~\cite{jiang2023llmlingua,xu2025cod}. Model-routing work such as FrugalGPT, RouteLLM, LLMCascades, and EllieSQL changes which model is called~\cite{chen2023frugalgpt,ong2024routellm,dohan2022cascades,zhu2025elliesql}. LangSelect uses the same routing abstraction but changes the action space: the model, prompt family, and executor remain fixed while the target programming language changes.

Multilingual code benchmarks and models make this action meaningful. HumanEval, MBPP, APPS, BigCodeBench, and DS-1000 provide executable coding tasks~\cite{chen2021codex,austin2021mbpp,hendrycks2021apps,zhuo2024bigcodebench,lai2022ds1000}; HumanEval-X, MultiPL-E, and McEval extend evaluation across languages~\cite{zheng2023codegeex,cassano2023multipl,chai2024mceval}. CodeGen, StarCoder, Code Llama, CodeGeeX, and MoLE show that modern models can generate across language portfolios~\cite{nijkamp2022codegen,li2023starcoder,roziere2023codellama,zheng2023codegeex,deng2025mole}. LangSelect treats that language axis as an inference-time control variable rather than only an evaluation label.

Programming-language confusion and tokenization bias explain why the decision cannot be a static shortest-language rule. Models can drift toward Python even under explicit non-Python instructions~\cite{moumoula2025plc}, while BPE and SentencePiece segment languages unevenly~\cite{sennrich2016bpe,kudo2018sentencepiece,petrov2023tokenizer}. Comparative language studies further show that concision differs for semantic reasons as well as tokenizer artifacts~\cite{nanz2015comparative,bostrom2020bpe}. LangSelect combines this observation with contextual-bandit machinery~\cite{li2010linucb,zhou2020neural} and verification/fallback, making language choice explicit and empirically checked inside coding-agent workflows~\cite{yang2024sweagent,xia2024agentless,zhang2024autocoderover,bouzenia2024repairagent,murali2023codecompose,pandey2024copilot}.

\paragraph{Cost-control axes.}
\phantomsection\label{app:related_comparison_text}
The cost-optimization methods above act on different parts of the inference stack, so their savings should be read as source-scoped rather than directly comparable. LLMLingua reports 20--40\% input-token savings for verbose prompts by compressing the prompt before generation~\cite{jiang2023llmlingua}. FrugalGPT, RouteLLM, and LLMCascades route or cascade across models, with reported savings of 98\% cost, 50\% cost, and 65\% GPU respectively under their own task and accounting settings~\cite{chen2023frugalgpt,ong2024routellm,dohan2022cascades}. Chain-of-Draft changes the candidate or reasoning budget and reports more than 50\% billing reduction in its setting~\cite{xu2025cod}. LangSelect is complementary to these methods: it keeps the model and prompt family fixed, changes the output programming language, and reports +50.3\% live harness-proxy savings with GPT-5 and 42.1\% verified replay headroom on language-flexible tasks.

\section{Extended Results}
\label{app:extended}

Figure~\ref{fig:regret} expands the convergence diagnostic into a full adaptive-regret comparison from the local experiment trace. This figure is a mechanism check: it shows whether the adaptive policy learns from verified feedback in a small trace. The paper's main empirical claims remain the held-out replay and live targeted-generation results in Section~\ref{sec:results}.

\begin{figure}[H]
  \centering
  \begin{tikzpicture}
    \begin{axis}[
      width=0.92\linewidth, height=5.0cm,
      xlabel={Number of Tasks $N$},
      ylabel={Cumulative Regret $R(N)$},
      xlabel style={yshift=2pt},
      ylabel style={xshift=2pt},
      xmin=0, xmax=330,
      ymin=0, ymax=70,
      xtick={0,50,100,150,200},
      ytick={0,20,40,60},
      ymajorgrids=true,
      xmajorgrids=true,
      grid style={dashed,gray!30},
    ]
      \addplot[thick,langblue] coordinates {
        (0,0)(10,0.89)(20,2.07)(30,2.18)(40,2.18)(50,2.75)(60,3.69)(70,4.87)(80,4.87)(90,4.87)(100,5.68)(110,5.73)(120,5.83)(130,5.9)(140,7.25)(150,8.35)(160,8.35)(170,9.35)(180,9.35)(190,9.56)(200,9.61)(204,9.61)
      };
      \addplot[thick,langorange] coordinates {
        (0,0)(10,1.92)(20,3.17)(30,3.79)(40,4.87)(50,6.91)(60,8.84)(70,10.87)(80,14.61)(90,17.87)(100,19.87)(110,21.88)(120,23.61)(130,26.51)(140,30.17)(150,33.17)(160,34.23)(170,34.26)(180,38.14)(190,40.35)(200,45.62)(204,45.62)
      };
      \addplot[thick,langgreen,dashed] coordinates {
        (0,0)(10,2.03)(20,2.23)(30,3.17)(40,5.21)(50,6.52)(60,7.46)(70,11.48)(80,16.03)(90,18.04)(100,20.88)(110,22.0)(120,22.09)(130,26.28)(140,27.36)(150,29.36)(160,32.15)(170,34.21)(180,37.3)(190,41.43)(200,45.5)(204,45.5)
      };
      \addplot[thick,langgray,dotted] coordinates {
        (0,0)(10,4.05)(20,6.01)(30,7.47)(40,9.47)(50,10.56)(60,12.76)(70,16.27)(80,23.65)(90,27.82)(100,30.87)(110,33.92)(120,36.92)(130,40.04)(140,43.12)(150,50.24)(160,53.33)(170,54.38)(180,58.3)(190,59.46)(200,64.27)(204,64.31)
      };
      \node[anchor=west, text=langblue] at (axis cs:207,9.61) {LangSelect};
      \node[anchor=west, text=langorange] at (axis cs:207,48.0) {LinUCB};
      \node[anchor=west, text=langgreen] at (axis cs:207,42.8) {$\varepsilon$-greedy};
      \node[anchor=west, text=langgray] at (axis cs:207,64.31) {Random};
    \end{axis}
  \end{tikzpicture}
  \caption{Cumulative regret diagnostic from the adaptive trace.}
  \label{fig:regret}
\end{figure}
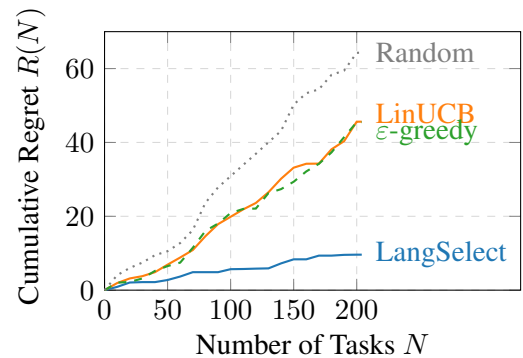

Figure~\ref{fig:regret} shows that Neural-LinUCB accumulates less regret than the adaptive LinUCB, $\varepsilon$-greedy, and random diagnostics in the 204-task trace. This supports the mechanism story, but the trace is too small for headline claims; those claims rely on held-out replay and live targeted generation.

Tables~\ref{tab:live_eval_main} and~\ref{tab:live_eval_cost} of the main text report the corresponding live-evaluation deployable savings under harness-proxy accounting ($n{=}450$). Any future cross-model claim should be regenerated from model-specific traces with the same provenance, harness, and token-count accounting used for the main results.

\section{Reproducibility Checklist}
\label{app:repro}

This checklist names the run artifacts and boundaries needed to reproduce the reported comparisons. It is intentionally narrower than a release checklist for production deployment: the paper evaluates language routing under fixed prompts, fixed harnesses, and fixed accounting rules.

\begin{itemize}
  \item \textbf{Dataset}: verified MultiLang-Bench corpus, 3{,}000 tasks, 8 languages, and a 2{,}100/450/450 train/validation/test split.
  \item \textbf{Language portfolio}: $\{$C, Clojure, Go, Java, Julia, Python, Ruby, Rust$\}$ for both replay and live targeted generation.
  \item \textbf{Generation model}: GPT-5 via OpenAI direct API for all live targeted-generation rows. Tokens counted with GPT-5's \texttt{tiktoken cl100k\_base} encoder.
  \item \textbf{Replay evidence}: replay rows choose among verified-corpus solutions and have $100\%$ pass@1 by construction; all live attempts are charged, including failed and fallback retries.
  \item \textbf{Feature models}: CodeBERT is frozen for embedding-based selectors; deployable metadata includes task metadata available before generation. Hindsight-derived or token-spread features are not used in main-policy rows.
  \item \textbf{Randomness}: split assignment, bootstrap resampling, and selector training use fixed seeds recorded in the run manifests. The manuscript reports only generated reports whose manifests match the released dataset hash.
  \item \textbf{Hyperparameters}: the main LinUCB-style selectors use a 256-dimensional task representation, validation-tuned confidence thresholds for the hybrid policy, and a maximum fallback budget of three attempts when recovery is evaluated.
  \item \textbf{Execution boundary}: all accepted solutions run under fixed local harnesses with network access disabled; harnesses, tests, wrappers, harness-supplied imports, comments, docstrings, and uniformly supplied signatures are excluded from function-body token accounting.
\end{itemize}

\end{document}